\documentclass{article}
\pdfoutput=1

\PassOptionsToPackage{numbers, compress}{natbib}
\usepackage[preprint]{neurips_2024}

\usepackage[utf8]{inputenc} % allow utf-8 input
\usepackage[T1]{fontenc}    % use 8-bit T1 fonts
\usepackage{hyperref}       % hyperlinks
\usepackage{url}            % simple URL typesetting
\usepackage{booktabs}       % professional-quality tables
\usepackage{amsfonts}       % blackboard math symbols
\usepackage{nicefrac}       % compact symbols for 1/2, etc.
\usepackage{microtype}      % micro typography
\usepackage{xcolor}         % colors
\usepackage{adjustbox}
\usepackage{diagbox}
\usepackage{multirow}
\usepackage{makecell}
\usepackage{subfigure}

\usepackage{times}
\usepackage{epsfig}
\usepackage{graphicx}
\usepackage{amsmath}
\usepackage{amssymb}
\usepackage{colortbl}
\usepackage{bbding}
\usepackage{ragged2e}
\usepackage{float}
\usepackage[capitalize]{cleveref}
\crefname{section}{Sec.}{Secs.}
\Crefname{section}{Section}{Sections}
\Crefname{table}{Table}{Tables}
\crefname{table}{Tab.}{Tabs.}
\newtheorem{theorem}{Theorem}
\newtheorem{proof}{Proof}

\title{Decomposing the Neurons: \\Activation Sparsity via Mixture of Experts for Continual Test Time Adaptation}

\author{%
Rongyu Zhang \\
Nanjing University \\
  \And
  Aosong Cheng$^\ast$ \\
  Peking University \\
    \And
  Yulin Luo\thanks{Equal contribution} \\
  Peking University \\
      \And
  Gaole Dai \\
    Peking University \\
      \And
      Huanrui Yang \\
UC, Berkeley \\
      \And
  Jiaming Liu \\
  Peking University \\
      \And
  Ran Xu \\
  Peking University \\
        \AND
  Li Du \\
  Nanjing University \\
        \And
  Yuan Du\thanks{Corresponding authors} \\
  Nanjing University \\
          \And
  Yanbing Jiang \\
  Peking University\thanks{School of Software and Microelectronics, Peking University} \\
          \And
  Shanghang Zhang$^\dagger$  \\
  Peking University\thanks{National Key Laboratory for Multimedia Information Processing in School of Computer Science at Peking
University (Same for all author from Peking University if not mentioned.)} \\
}

\begin{document}

\maketitle

\begin{abstract}
Continual Test-Time Adaptation (CTTA), which aims to adapt the pre-trained
model to ever-evolving target domains, emerges as an important task for vision models. As current vision models appear to be heavily biased towards texture, continuously adapting the model from one domain distribution to another can result in serious catastrophic forgetting. Drawing inspiration from the human visual system's adeptness at processing both shape and texture according to the famous Trichromatic Theory, we explore the integration of a Mixture-of-Activation-Sparsity-Experts (MoASE) as an adapter for the CTTA task. Given the distinct reaction of neurons with low/high activation to domain-specific/agnostic features, MoASE decomposes the neural activation into high-activation and low-activation components with a non-differentiable Spatial Differentiate Dropout (SDD).
Based on the decomposition, we devise a multi-gate structure comprising a Domain-Aware Gate (DAG) that utilizes domain information to adaptive combine experts that process the post-SDD sparse activations of different strengths, and the Activation Sparsity Gate (ASG) that adaptively assigned feature selection threshold of the SDD for different experts for more precise feature decomposition. Finally, we introduce a Homeostatic-Proximal (HP) loss to bypass the error accumulation problem when continuously adapting the model. Extensive experiments on four prominent benchmarks substantiate that our methodology achieves state-of-the-art performance in both classification and segmentation CTTA tasks. Our code is now available at \url{https://github.com/RoyZry98/MoASE-Pytorch}.
\end{abstract}

\section{Introduction}
With the emergence of deep-learning-based methods in autonomous driving\cite{yurtsever2020survey,liu2022multi,chi2023bev,teng2023motion,yang2023bevformer} and robotics\cite{chen2023prime,li2023manipllm,li2023imagemanip,hu2023planning,gao2023textdeformer}, the continuously changing test-time scenarios in these applications raise significant challenges for stationary machine perception systems\cite{zhang2024efficient,han2020occuseg,tian2024occ3d,xie2021segformer} which anticipates that the test-time data distribution always mirrors that of the training data, resulting in severe error accumulation and catastrophic forgetting. To address this issue, Continual Test-Time Adaptation (CTTA) has been proposed \cite{wang2022continual,song2023ecotta}, which moves beyond the conventional setting of Test-Time Adaptation (TTA) that handles a single shift \cite{wang2020tent,liang2023comprehensive,brahma2023probabilistic} to manage a sequence of distribution shifts over time. Current CTTA approaches \cite{liu2023vida,yang2023exploring,gan2022decorate,liu2023adaptive} predominantly utilize a teacher-student framework to concurrently extract domain knowledge in target domains by generating pseudo-labels. However, such methods frequently struggle to identify and differentiate domain-specific and domain-agnostic features with \texttt{implicit}\cite{yang2023exploring,liu2023vida} self-training visual prompt and high/low-rank adapter, which lacks interpretability and reduced controllability over the training process. Moreover, in contrast to implicit models, vision science reveals that the human visual system employs a clearly defined, explicit mechanism with an absolute threshold \cite{barlow1956retinal,koenig2011absolute} to process visual signals. Therefore, we aim to explore the solution for CTTA tasks from an \texttt{explicit} perspective to decompose the feature representations for better perception.

\begin{figure*}[t]
% \vspace{-0.2cm}
\centering
\includegraphics[width=0.99\linewidth]{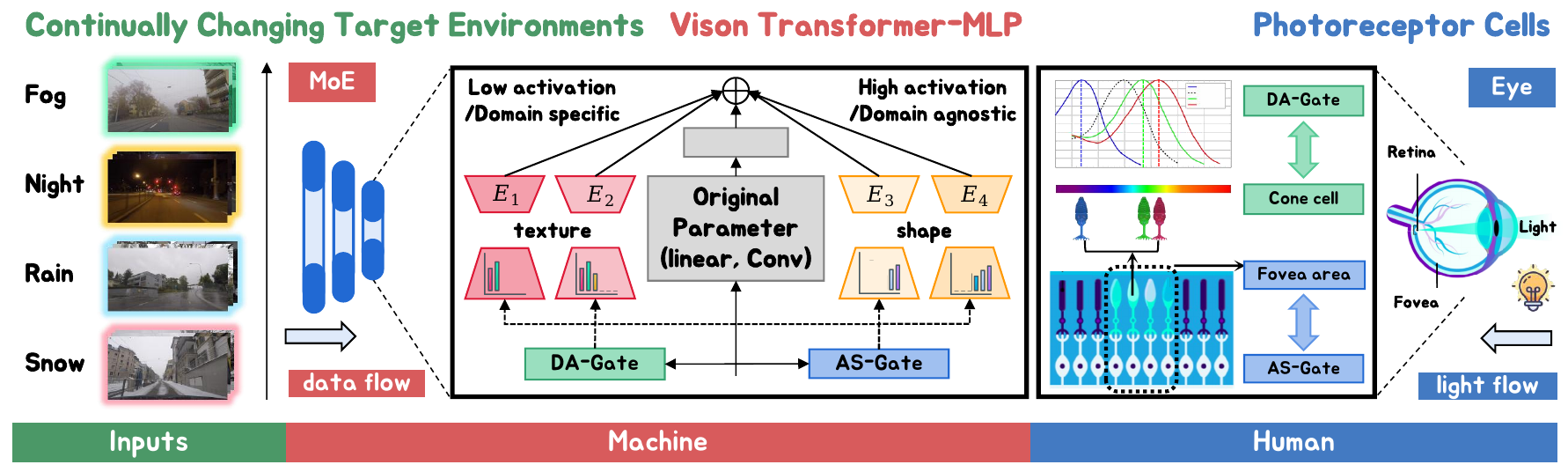}
% \vspace{-0.2cm}
\caption{\textbf{The problem and motivation.} Our goal is to effectively adapt the source pre-trained model to continually changing target domains. We propose a Mixture-of-Experts (MoE) based to encode the different features of texture and shape with different experts. Our design is inspired by the photoreceptor cells in the human visual system, where the three types of cone cells are sensitive to different wavelengths of light and the Fovea.}
\label{fig: teaser}
\vspace{-0.5cm}
\end{figure*}

The renowned Trichromatic Theory\cite{von1867handbuch} demonstrates that the Retina in the human eye consists of different types of cone cells\cite{mustafi2009structure}, each sensitive to specific light wavelengths. This photoreceptor
cells are densely packed in the Fovea\cite{bringmann2018primate}, an area highly sensitive to detail and essential for high-resolution vision tasks to enable self-attention, while other cone cells work with additional rod cells to manage perception in less focused areas of the visual field as well as the dimly lit environment, as shown in the blue part of the \cref{fig: teaser}. Similarly, deep neural networks (DNN) also exhibit parallel characteristics which have been proven by previous works\cite{li2024emergence,zhang2024multi,yang2020fda,xu2021fourier} that strongly activated neuron encodes shape and structure, which are domain agnostic, while weak activation corresponds to domain-specific texture and style. Therefore, we can delineate the conceptual associations between \textbf{\{low/high activation\}-\{domain specific/agnostic\}-\{texture/shape\}.} Drawing inspiration from the human visual system and the prior research, we propose an intriguing hypothesis: $\spadesuit$ \textbf{Can we explicitly decompose neurons by activation degree to encode shapes and textures separately for better model perception to differentiate the continuously changing environments?}

To demystify the role of activation sparsity in CTTA, we manually decompose strongly and weakly activated neurons in a pretrained and visualize their responses to inputs from varying domains in \cref{fig:cam}. We observe a clear distinction in the neuron attention, where \textbf{strongly activated neurons} focus on \textbf{domain-agnostic foreground features} relating to the main object; while \textbf{weakly activated neurons} attend to \textbf{domain-specific background features} of styles and noises. 
This motivates our study on the \texttt{explicit} decomposition of neural activation with a sparse Mixture-of-Experts (MoE)\cite{jacobs1991adaptive,masoudnia2014mixture,riquelme2021scaling,liu2024intuition,zhou2022mixture} architecture to mimic the partitioned processing of visual signals by the Fovea and peripheral regions in the Retina\cite{field2005retinal}. We propose the Mixture-of-Activation-Sparsity-Experts (MoASE), an adapter integrated into pre-trained models, featuring a non-differentiable Spatial Differentiate Dropout (SDD) mechanism. This setup enhances the extraction of domain-agnostic object shapes and structures and identifies domain-specific textures from a spatial-wise perspective, such as weather-related noise from fog and snow\cite{zhang2024efficient}  through specialized expert modules, whose functionality is highlighted in the red section of \cref{fig: teaser}.

% We found that highly activated neurons predominantly encode domain-agnostic structural elements such as shapes and contours, while low activation responses mainly capture domain-specific textural and color features, including weather-related noise like fog and snow.

Moreover, we developed a multi-gate module to enhance dynamic perception in the fluctuating CTTA scenario\cite{wang2022continual,song2023ecotta,liu2023vida,yang2023exploring,gan2022decorate,liu2023adaptive}, comprising the Domain Aware Gate (DAG) to inform the experts with domain-specific information and the Activation Sparsity Gate (ASG) to adjust the threshold of activation selection for each expert. Moreover, to mitigate error accumulation caused by the random initialization of the injected MoASE, we have devised a Homeostatic-Proximal (HP) loss to regularize and balance the updates of domain-specific and domain-agnostic parameters.

Extensive experiments demonstrate the superiority of our proposed Mixture-of-Activation-Sparsity-Experts (MoASE) across three image classification benchmarks\cite{krizhevsky2009learning,hendrycks2019benchmarking} and one segmentation benchmark\cite{cordts2016cityscapes,sakaridis2021acdc} on CTTA scenarios with improvements exceeding 15.3\% in classification accuracy and 5.5\% in segmentation mIoU.
% These results underscore MoASE's robust capability to adapt and perform reliably across diverse and dynamic environments. 
The major contribution of our paper can be summarized as follows:
    \begin{itemize}
        \item We draw inspiration from the human visual system to develop a Mixture-of-Sparsity-Activation-Experts (MoASE) model, which addresses the issues of error accumulation and catastrophic forgetting to face the continuously changing distribution.
        \item We decompose the neuron-encoded activation into domain-specific and domain-agnostic features, using distinct expert models to encode texture and shape independently with Spatial Differentiate Dropout (SDD).
        \item We developed a multi-gate module featuring the Domain Aware Gate (DAG) and Activation Sparsity Gate (ASG), utilizing domain information to dynamically generate adaptive routing strategies and activation thresholds for experts.
        \item We propose a novel Homeostatic-Proximal (HP) loss to mitigate the accumulation error and enhance performance within the teacher-student framework. Extensive experiments demonstrate the efficacy of our MoASE across diverse CTTA scenarios.
    \end{itemize}

\section{Related works}
\textbf{Continual Test-Time Adaptation (CTTA)} addresses the challenge of adapting to a non-static target domain, which complicates traditional TTA methods. The pioneering work by Wang et al. \cite{wang2022continual} combined bi-average pseudo labels with stochastic weight resets to tackle this issue. To mitigate error accumulation, Ecotta \cite{song2023ecotta} employs a meta-network for output regularization. While these approaches focus on model-level solutions, other studies \cite{gan2023decorate,yang2023exploring, ni2023distribution,liu2023vida} explore the use of visual domain prompts or minimal parameter adjustments for continual learning. Liu \cite{liu2023adaptive} introduced reconstruction techniques for continual adaptation, and BECotta \cite{lee2024becotta} implemented a Mixture of Experts strategy in CTTA, promoting effective domain-specific knowledge retention. Unlike previous \texttt{implicit} methods, our MoASE adopts an \texttt{explicit} approach to this challenge.

\textbf{Activation Sparsity} refers to the presence of numerous weakly-contributing elements in activation outputs\cite{Chen_2023_CVPR,kurtz2020inducing,yang2019dasnet,yang2019sparse}. SparseViT\cite{song2024prosparse} revisits this concept for modern window-based vision transformers, aiming to increase speed and reduce computation. Grimaldi\cite{Grimaldi_2023_ICCV} introduces semi-structured activation sparsity that can be leveraged with minor runtime adjustments to significantly enhance speed on embedded devices. Meanwhile, Mirzadeh\cite{mirzadeh2023relu} explores the reuse of activated neurons in LLMs, proposing strategies to reduce computation. However, all the previous works aim to improve model efficiency until \cite{zhang2024multi,li2024emergence} reveal the contribution of shape bias for model performance.

\textbf{Mixture-of-Experts (MoE)} is initially introduced by Jacobs and Jordan \cite{jacobs1991adaptive,jordan1994hierarchical}, uses independent modules to boost expressiveness and cut computational costs. Eigen and Ma \cite{eigen2013learning,ma2018modeling} evolved it into the MoE layer. In natural language processing, GShard \cite{gshard} and Switch Transformer \cite{switch} incorporated MoE with top-1/2 routing to enhance capacity. Fixed routing \cite{hash,stablemoe} and ST-MoE \cite{st_moe} aimed to stabilize training. Recent developments \cite{zhu2023sira,liu2024intuition} introduced efficient adaptors within MoE and \cite{zhao2023moec} combined MoE with implicit neural network for image compression. In computer vision, M$^{3}$ViT by Liang et al. \cite{liang2022m} selectively engages experts, while Zhang et al. \cite{zhang2024efficient} merged feature modulation with MoE for better image restoration. Our MoASE leverages a multi-expert setup to manage diverse neuronal activation.

\begin{figure*}[t]
% \vspace{-0.2cm}
\centering
\includegraphics[width=0.99\linewidth]{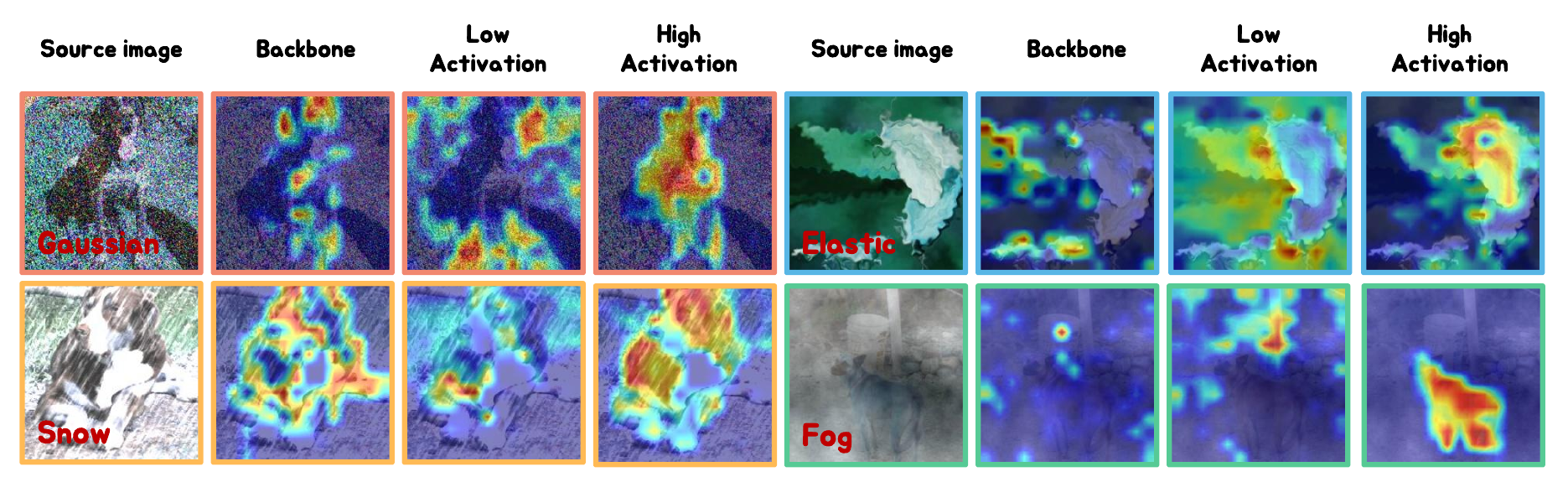}
% \vspace{-0.2cm}
\caption{\textbf{The visualization analysis of the Class Activation Map (CAM).} We adopt CAM to compare the attention of the low-activation MoASE, high-activation MoASE, and the original model during the continual adaptation process.}
\label{fig:cam}
% \vspace{-0.3cm}
\end{figure*}

\section{Motivation}
\label{sec:motivation}
Drawing inspiration from the complexities of the human visual system \cite{von1867handbuch,mustafi2009structure,bringmann2018primate} and reinforced by seminal findings in recent research \cite{li2024emergence,zhang2024multi,yang2020fda,xu2021fourier}, we come up with the hypothesis of leveraging activation sparsity via the innovative integration of a parallel MoE for CTTA. This MoE employs a network of architecture-consistent experts, combined with the Spatial Differentiate Dropout technique, to effectively decompose activation neurons. Such a configuration not only enhances the model's capability to perceive domain-agnostic object structures but also sharpens its accuracy in identifying domain-specific knowledge within dynamically changing environments.

To provide empirical support for our hypothesis, we extended our analysis to include a qualitative evaluation using Class Activation Mapping (CAM)\cite{zhou2016learning} within the ImageNet-to-ImageNet-C CTTA scenario. As depicted in \cref{fig:cam}, our study explores feature representations in various target domains, each characterized by different activation intensities, including Gaussian and Elastic noise, as well as Snow and Fog weather conditions. Specifically, we selectively retained only the high or low activation values within the ViT-base model under the CTTA scenario\cite{wang2022continual}. Our findings reveal that models maintaining only weakly activated neurons primarily accentuate fluctuations in background noise while overlooking the foreground object. This implies that weakly activated neurons are adept at capturing domain-relevant information. Conversely, models with strongly activated neurons exhibit a contrasting pattern, focusing more intensively on object shapes and structures. This observation supports the notion that strongly activated neurons are sensitive to domain-agnostic object structures, thus validating previous visual science principles and our hypothesis.

\section{Methods}
\label{sec:method}

\subsection{Preliminary}
\textbf{Continual Test-Time Adaptation.} We pre-train the model $\theta(y_{s}|x_{s})$ on source domain $D_{S}={(\mathcal{Y}_S, \mathcal{X}_S)}$ and adapt it to multiple target domains $D_{T_i}={\{\mathcal{X}_{\mathcal{T}_i}\}}_{i=1}^{n}$, where $n$ indicates the number of target datasets. Utilizing the robustness of mean teacher predictions \cite{tarvainen2017mean, dobler2023robust}, we implement a teacher-student framework ($\theta^{\mathcal{T}}$ and $\theta^{\mathcal{S}}$) to maintain stability during adaptation\cite{wang2022continual, gan2022decorate}. This adaptation process, which is unsupervised and single-pass for target domain data $x\in \mathcal{X}_{T}$, does not require access to source domain data. We aim to adapt the pre-trained model to continuously evolving target domains while maintaining recognition capabilities on familiar distributions. The framework and methodology are outlined in \cref{fig:framework}.

\textbf{Mixture-of-Experts.}
The MoE model is fundamentally composed of $E$ expert functions $e_{i}: \mathcal{X}_{\mathcal{T}} \rightarrow \mathbb{R}^{p}$ for $i\in E$, alongside a trainable gating mechanism $g: \mathcal{X}_{\mathcal{T}} \rightarrow \mathbb{R}^{q}$ which allocates inputs to experts by outputting a probability vector. For an input sample $x$, the MoE's output is the aggregate of expert contributions, each weighted by the router's assigned probabilities, which can be mathematically represented with soft routing as:
\begin{equation}
\label{equ:routing}
    \begin{aligned}
y = \sum_{i=1}^{E}e_{i}(x)g_{i}(x), \ \ g(x)=\sigma(\boldsymbol{A}x+\boldsymbol{b}) \quad
s.t. \ \  g(x)\geq0 \ \ \text{and} \ \ \sum_{i=1}^{E}g_{i}(x)=1
    \end{aligned}
\end{equation}
where $\sigma(\cdot)$ signifies the softmax function, $\boldsymbol{A}\in\mathbb{R}^{n\times d}$ represents a matrix of trainable weights, and $\boldsymbol{b}\in \mathbb{R}^{n}$ is the bias vector. This gate operates densely, allocating nonzero probabilities to all experts. 

\begin{figure*}[t]
% \vspace{-0.2cm}
\centering
\includegraphics[width=0.99\linewidth]{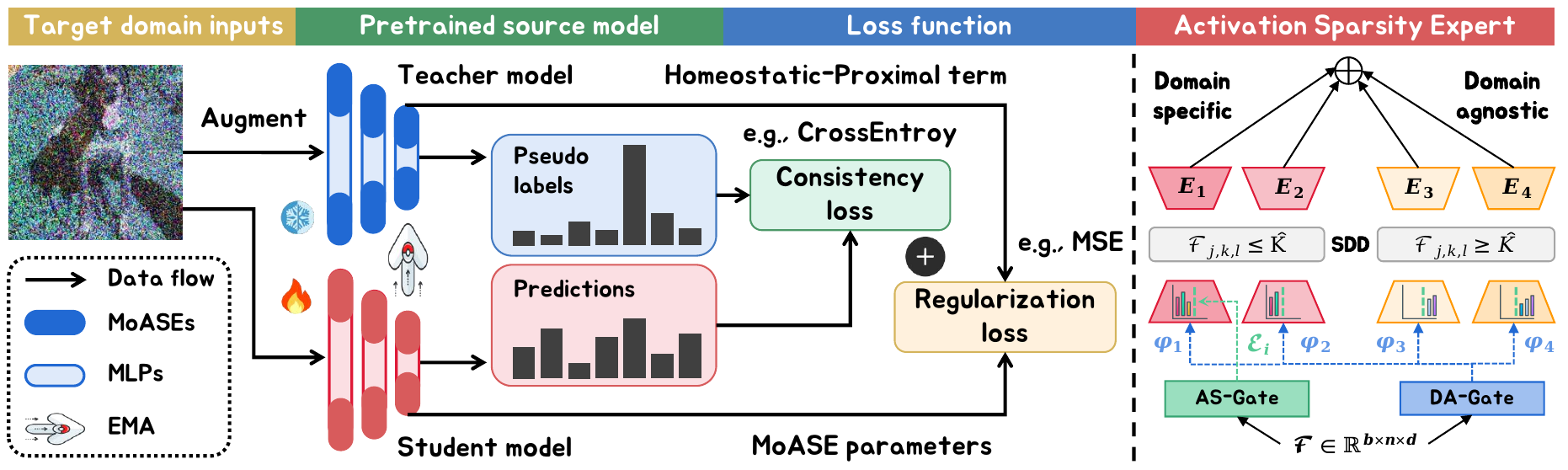}
% \vspace{-0.2cm}
\caption{\textbf{The overall framework of Mixture-of-Activation-Sparsity-Experts (MoASE).} (Left) We integrate the MoASE into the linear layers of a pre-trained source model with a teacher-student framework, using consistency loss and a specially formulated Homeostatic-Proximal (HP) loss as the optimization target.
(Right) Depending on the degree of distribution shift, we devised a multi-router structure that includes the Domain Aware Gate (DAG) and the Activation Sparsity Gate (ASG).} 
\label{fig:framework}
\vspace{-0.3cm}
\end{figure*}

\subsection{Mixture-of-Activation-Sparsity-Experts}

\textbf{Spatial Differentiate Dropout.} We have adopted the sparse coding principle by incorporating a spatially oriented Top-K dropout operation in our framework. The MoASE architecture follows the same as previous MoE works\cite{jacobs1991adaptive,jordan1994hierarchical} with a gate $g(\cdot)$ and multiple experts $e(\cdot)$ with two linear layers. However, unlike traditional dropout layer between the linear layers that randomly discard $p(\%)$ of neuron activation on channel dimension, our MoASE incorporates the innovative Spatial Differentiate Dropout (SDD) that selectively retains the top/bottom $q(\%)$ significant responses from the spatial-wise token dimension following the non-linear activation function (e.g., ReLU) since the domain-related information was distributed across each token. Specifically, for an input feature $\mathcal{F}\in \mathbb{R}^{b\times n\times d}$, the SDD layer $\delta(\cdot)$ receive the non-linear-transformed activated feature and generate the $\hat{\mathcal{F}}\in \mathbb{R}^{b\times n\times d}$ with the sorting function $\delta=argsort(\mathcal{F}, K, L)$:
\begin{equation}
\label{equ:topk}
\hat{\mathcal{F}}[j,k,l]:=\left\{
\begin{aligned}
  \ \mathcal{F}[j,k,l], \quad &\text{if} \ \ \mathcal{F}[j,k,l]\ \geq \ argsort(\mathcal{F}[j,:,:])[K] \ \ \text{and} \ \ L=True  \\
  \ \mathcal{F}[j,k,l], \quad  &\text{if} \ \ \mathcal{F}[j,k,l]\ \leq \ argsort(\mathcal{F}[j,:,:])[K] \ \ \text{and} \ \ L=False  \\
  \ 0, \quad &\text{otherwise}
\end{aligned}
\right.
\end{equation}
where $K=\lfloor nd\times q\rfloor$ is the index, $L$ is the boolean variable in the $argsort(\cdot)$ function indicating whether to select the largest or smallest values. For the domain-agnostic experts $L$ is set to $True$, while for the domain-specific experts, $L$ is set to $False$. To ensure that different experts can precept different levels of activation, we set the threshold according to the number of experts as hyperparameters which are detailed in \cref{ap:details}.

\textbf{Domain Aware Gate and Activation Sparsity Gate.}
To enhance dynamic perception in the fluctuating CTTA scenario \cite{wang2022continual,song2023ecotta,liu2023vida,yang2023exploring,gan2022decorate,liu2023adaptive}, we propose a multi-gate module consisting of the Domain Aware Gate (DAG) and the Activation Sparsity Gate (ASG). These gates provide domain-specific information and differentiate domain-agnostic objects for each expert. Both gates employ the architecture of $g(\cdot)$ described in \cref{equ:routing}, but serve distinct proposes. Specifically, the DAG determines the routing of input tokens across the experts. To better help the router adapt to diverse domains in the CTTA scenario, we specifically decompose low activations from the input feature $\mathcal{F} \in \mathbb{R}^{b \times n \times d}$ using the SDD layer with $q\%=\frac{1}{2}$,to let the router weight $G$ focuses more on domain-specific information. Simultaneously, the ASG uses the full input to generate a dynamic threshold $T\in \mathbb{R}^{1\times E}$ for SDD in each expert to facilitate adaptive activation decomposition in response to ongoing changes. The mathematics formulation can be described as: 
\begin{equation}
\label{equ:router}
\begin{aligned}
G = DAG(\delta(\mathcal{F}, \lfloor \frac{nd}{2}\rfloor, False)), \quad \quad T=ASG(\mathcal{F})
\end{aligned}
\end{equation}
where $G\in \mathbb{R}^{b\times n \times E}$ is the expert weight. The thresholds generated by ASG are combined with predefined thresholds in \cref{equ:topk} to compute the final $\hat{K}_{i} = \lfloor nd\times (q + \eta \cdot T_{i}) \rfloor$ for each expert, where $\eta$ is a temperature scale that ensures $\hat{K}$ does not exceed the upper limit. This integration supports a responsive and balanced adaptation to the evolving scenario demands. 

\subsection{Optimization objective}
\textbf{Homeostatic-Proximal loss.} Building on prior CTTA research ~\cite{wang2022continual}, we employ the teacher model $\theta^\mathcal{T}$ to generate pseudo labels $y_{pd}$ to minimize the task-specific training objective $\mathcal{L}_{TS}(\cdot)$, which are used to update MoASE. Despite the successful implementation of multiple experts to perceive various degrees of activation, we also strive to maintain statistical homeostasis amid continuous domain shifts by constraining student updates $\theta^{\mathcal{S}}$ to stay closely aligned with the initial (teacher) model $\theta^{\mathcal{T}}$, thus eliminating the need for manual intervention. In particular, instead of just minimizing the $\mathcal{L}_{TS}(\cdot)$, we further introduce the Homeostatic-Proximal term
to the original loss\cite{wang2022continual}
to approximately minimize the following optimization objective $\mathcal{L}_{Overall}$:
\begin{equation}
\mathcal{L}_{Overall} = \mathcal{L}_{TS}(\theta^{\mathcal{S}}) + \frac{\mu}{2}\sum_{e=1}^{E}||\theta_{e}^{\mathcal{S}}-\theta_{e}^{\mathcal{T}}||^{2}.
\label{apeq:hploss}
\end{equation}
where $\mu$ is a hyperparameter,  $\theta_{e}^{\mathcal{S}}$ and $\theta_{e}^{\mathcal{T}}$ represent the parameters of each expert in the MoASE teacher-student framework, respectively.

\textbf{Exponential Moving Average.} Moreover, as for the update of the teacher model, we initialize both models with source pre-trained parameters and use the Exponential Moving Average (EMA) approach to update the teacher model $\theta^{\mathcal{T}}$ following established practices \cite{gan2023decorate}:
\begin{equation}
 \theta^{\mathcal{T}}_{t} = \alpha \times\theta^{\mathcal{T}}_{t-1} + (1-\alpha) \times\theta^{\mathcal{S}}_{t}
\label{eq:ema}
\end{equation}
In this setup, $t$ indicates the time step, and we set the update weight $\alpha = 0.999$ \cite{AnttiTarvainenetal2017}.

\begin{figure*}[t]
\centering
\includegraphics[width=0.99\linewidth]{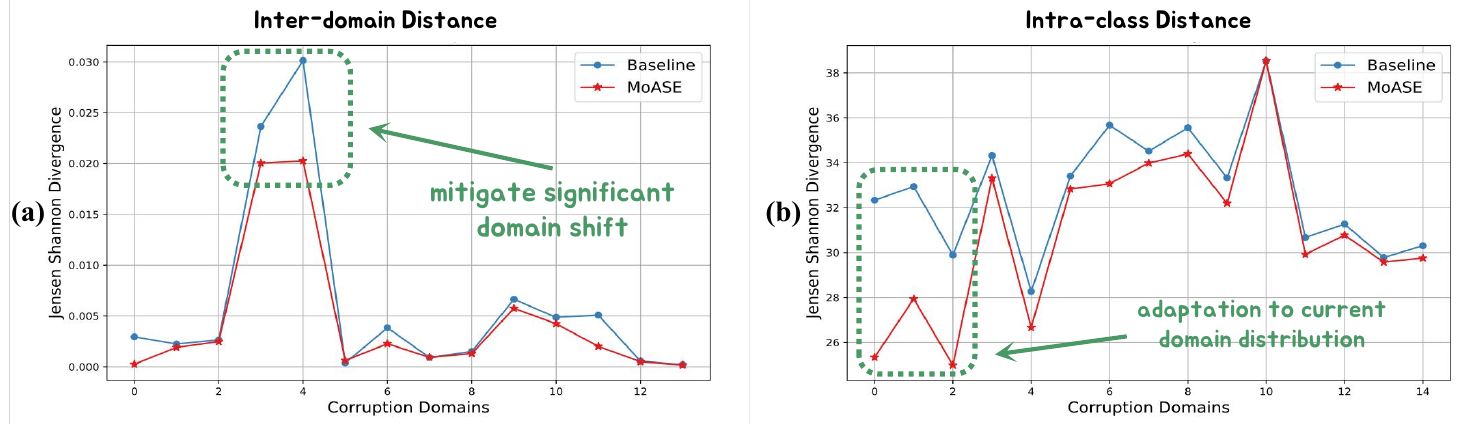}
\caption{\textbf{Inter-domain and intra-class distance.} The X-axis displays the 15 corruption domains in CIFAR10-C, listed in sequential order. (a) MoASE more effectively reduces inter-domain divergence than the source model across all 14 domain shifts. (b) MoASE significantly improves intra-class feature aggregation, producing results that closely align with those of our proposed method.}
\label{divergence}
% \vspace{-0.2cm}
\end{figure*}

\section{Justification}
\label{sec:just}
To substantiate our hypothesis, we measure the domain representation of MoASE by calculating the distribution distance using Ben-David's domain distance definition \cite{ben2006analysis, ben2010theory} and the $\mathcal{H}$-divergence metric, building on previous domain transfer studies \cite{ganin2016domain}. The $\mathcal{H}$-divergence between source domain $D_S$ and target domain $D_{T_i}$ is given as follows and error bound analysis can be found in \cref{ap:bound}:
\begin{equation}
d_\mathcal{H}(D_S, D_{T_i}) = 2 \mathop{\mathrm{sup}}_{\mathcal{D} \sim \mathcal{H}} | \mathop{\mathrm{Pr}}_{x \sim D_S}[\mathcal{D}(x)=1] - \mathop{\mathrm{Pr}}_{x \sim D_{T_i}}[\mathcal{D}(x)=1] |
\label{apeq:dh}
\end{equation}
where $\mathcal{H}$ is the hypothesis space and $\mathcal{D}$ the discriminator. Consistent with the methodologies presented in \cite{ruder2017learning,allaway2021adversarial}, we employ the $Jensen$-$Shannon\ (JS)\ divergence$ between two adjacent domains as an approximation of $\mathcal{H}$-divergence, owing to its demonstrated efficacy in differentiating between domains. A relatively small inter-domain divergence suggests that the feature representation is robust and exhibits reduced susceptibility to cross-domain shifts, as elaborated in \cite{ganin2016domain}:
\begin{equation}
JS(P_{D_S} || P_{D_{T_i}}) = \frac{1}{2}KL(P_{D_S}|| \frac{P_{D_S}+P_{D_{T_i}}}{2}) + \frac{1}{2}KL(P_{D_{T_i}}|| \frac{P_{D_S}+P_{D_{T_i}}}{2}),
\label{eq:js}
\end{equation}
where $P$ denotes probability distribution of model output features and $KL(\cdot||\cdot)$ denotes the $Kullback-Leibler\ (KL)\ divergence$. 
%\begin{equation}
%KL(P_1||P_2) = \sum_{i=0}^{n} P_1(x_i) log(\frac{P_1(x_i)}{P_2(x_i)})
%\label{eq:kl}
%\end{equation} 
As shown in \cref{divergence}(a), the MoASE shows significantly lower inter-domain divergence, indicating robust feature representation across domains. Moreover, we evaluate domain representation based on intra-class divergence, inspired by $k$-means clustering \cite{macqueen1967classification}. The mathematical representation can be formulated as:
\begin{equation}
IC = \frac{1}{|C|} \sum_{f_i \sim C} ||f_i - \frac{1}{|C|} \sum_{f_i \sim C} f_i||_2^2
\label{apeq:kmeans}
\end{equation}
where $f_{i}$ is the encoder output feature from each domain. Smaller intra-class divergence indicates a superior understanding of the domain, as depicted for the in \cref{divergence}(b). The substantial and consistent margin between MoASE and the baseline model, as shown in \cref{divergence}, demonstrates the effectiveness of our proposed method in mitigating domain shifts.

\section{Experiments}
\label{experiment}
% In this section, we evaluate the performance of our proposed Mixture-of-Activation-Sparsity-Experts (MoASE) on three image classification and one segmentation benchmarks under continual test-time adaptation scenarios. The results demonstrate the superiority of MoASE by gaining the improvement on various situation against previous state-of-the-arts methods by over x\% in classification accuracy and x\% segmentation mIoU.

% \subsection{Task settings and implementation}
% \label{sec:4.1}

\noindent\textbf{CTTA Task setting.}
Following \cite{wang2022continual}, in classification CTTA tasks, we adapt the pre-trained source model sequentially across fifteen target domains on CIFAR10-C, CIFAR100C, and ImageNet-C exhibiting the highest level of corruption severity (level 5). 
% The model's online prediction capabilities are evaluated immediately upon processing the input data.
In the context of segmentation CTTA, as per \cite{yang2023exploring,liu2023vida}, we employ an off-the-shelf model pre-trained on the Cityscapes dataset. For the continual adaptation to target domains, we utilize the ACDC dataset, which comprises images captured under four distinct adverse weather conditions: Fog→Night→Rain→Snow.
% To mimic real-life scenarios of environmental change, we cyclically expose the model to these conditions in a repeated sequence (Fog→Night→Rain→Snow) across multiple cycles.

\noindent\textbf{Implementation Details.} 
% In our CTTA experiments, we meticulously adhere to the implementation protocols established in previous research \cite{wang2022continual} to ensure both consistency and comparability across our studies. 
For the backbone architectures in classification CTTA, we employ ViT-base \cite{dosovitskiy2020image}, where we standardize input image sizes to 384$\times$384 for CIFAR and 224$\times$224 pixels for ImageNet. In segmentation CTTA, the Segformer-B5 model \cite{xie2021segformer}, pre-trained, serves as our source model, with input dimensions reduced from 1920x1080 to 960x540 for processing in target domains. $\eta$ is set to 0.1. Optimization utilizes the Adam algorithm \cite{kingma2014adam} with $(\beta_1, \beta_2) = (0.9, 0.99)$. Specific learning rates are assigned to each task: 1e-4 for classification, and 2e-4 for segmentation. 

\begin{table*}[t]
% \vspace{-0.1cm}
\caption{\label{tab:imagenet}\textbf{Classification error rate(\%) for ImageNet-to-ImageNet-C online CTTA task.} Gain(\%) represents the percentage of improvement in model accuracy compared with the source method.}
% \vspace{-0.3cm}
\centering
\setlength\tabcolsep{2pt}%调列距
\begin{adjustbox}{width=1\linewidth,center=\linewidth}
\begin{tabular}{c|c|ccccccccccccccc|cc}
\toprule
% Time & \multicolumn{15}{l|}{$t\xrightarrow{\hspace*{13.5cm}}$}& \\ \midrule
Method & Venue &
 \rotatebox[origin=c]{70}{Gaussian} & \rotatebox[origin=c]{70}{shot} & \rotatebox[origin=c]{70}{impulse} & \rotatebox[origin=c]{70}{defocus} & \rotatebox[origin=c]{70}{glass} & \rotatebox[origin=c]{70}{motion} & \rotatebox[origin=c]{70}{zoom} & \rotatebox[origin=c]{70}{snow} & \rotatebox[origin=c]{70}{frost} & \rotatebox[origin=c]{70}{fog}  & \rotatebox[origin=c]{70}{brightness} & \rotatebox[origin=c]{70}{contrast} & \rotatebox[origin=c]{70}{elastic\_trans} & \rotatebox[origin=c]{70}{pixelate} & \rotatebox[origin=c]{70}{jpeg}
& Mean$\downarrow$ & Gain$\uparrow$ \\
\midrule
\multicolumn{19}{c}{CIFAR10 \quad $\Rightarrow$ \quad CIFAR10-C}\\
\midrule
Source\cite{dosovitskiy2020image} & ICLR2021 &60.1&53.2&38.3&19.9&35.5&22.6&18.6&12.1&12.7&22.8&5.3&49.7&23.6&24.7&23.1&28.2&0.0\\
TENT\cite{DequanWangetal2021}  & CVPR2021 &57.7&56.3&29.4&16.2&35.3&16.2&12.4&11.0&11.6&14.9&4.7&22.5&15.9&29.1&19.5&23.5&+4.7\\
CoTTA\cite{wang2022continual} & CVPR2022 &58.7&51.3&33.0&20.1&34.8&20&15.2&11.1&11.3&18.5&4.0&34.7&18.8&19.0&17.9&24.6&+3.6\\
VDP\cite{gan2023decorate} & AAAI2023 &57.5&49.5&31.7&21.3&35.1&19.6&15.1&10.8&10.3&18.1&4.0&27.5&18.4&22.5&19.9&24.1&+4.1\\
ViDA\cite{liu2023vida} & ICLR2024 & \textit{52.9} &\textit{47.9} &\textbf{19.4} &\textbf{11.4} &\textit{31.3} &\textbf{13.3}&\textbf{7.6} &\textbf{7.6} &\textit{9.9} &\textit{12.5} &\textit{3.8} &\textit{26.3} &\textit{14.4} &\textit{33.9} &\textit{18.2} &\textit{20.7} &+\textit{7.5} \\
% Mask-CTTA\cite{liu2023adaptive} & CVPR2024 & 47.7 & 42.5 & 42.9 & 52.2 & 56.9 & 45.5 & 48.9 & 38.9 & 42.7 & 40.7 & 24.3 & 52.8  & 49.1 & 33.5 & 33.1 & 43.4 & +12.4 \\
\rowcolor{green!10}\textbf{Ours} & \textbf{Proposed} & \textbf{43.7}& \textbf{31.3}&\textit{25.1}&\textit{16.5}& \textbf{28.1}&\textit{13.8}&\textit{9.7}&\textit{8.3}& \textbf{7.1}& \textbf{10.1}&\textbf{3.0}&\textbf{12.9}& \textbf{12.0}& \textbf{16.3}& \textbf{13.5}& \textbf{16.8}& +\textbf{11.4}\\
\midrule
\multicolumn{19}{c}{CIFAR100 \quad $\Rightarrow$ \quad CIFAR100-C}\\
\midrule
Source\cite{dosovitskiy2020image} & ICLR2021 &55.0&51.5&26.9&24.0&60.5&29.0&21.4&21.1&25.0&35.2&11.8&34.8&43.2&56.0&35.9&35.4&0.0\\
TENT\cite{DequanWangetal2021}  & CVPR2021 &53.0&47.0&24.6&22.3&58.5&26.5&19.0&21.0&23.0&30.1&11.8&25.2&39.0&47.1&33.3&32.1&+3.3\\
CoTTA\cite{wang2022continual} & CVPR2022 &55.0&51.3&25.8&24.1&59.2&28.9&21.4&21.0&24.7&34.9&11.7&31.7&40.4&55.7&35.6&34.8&+0.6\\
VDP\cite{gan2023decorate} & AAAI2023 &54.8&51.2&25.6&24.2&59.1&28.8&21.2&20.5&23.3&33.8&\textbf{7.5}&\textbf{11.7}&32.0&51.7&35.2&32.0&+3.4\\
ViDA\cite{liu2023vida} & ICLR2024 & \textit{50.1} &\textit{40.7} &\textit{22.0} &\textbf{21.2} &\textit{45.2} &\textbf{21.6} &\textit{16.5} &\textbf{17.9} &\textbf{16.6} &\textit{25.6} &\textit{11.5} & 29.0 &\textit{29.6} &\textit{34.7} &\textbf{27.1} &\textit{27.3} &+\textit{8.1} \\
% Mask-CTTA\cite{liu2023adaptive} & CVPR2024 & 47.7 & 42.5 & 42.9 & 52.2 & 56.9 & 45.5 & 48.9 & 38.9 & 42.7 & 40.7 & 24.3 & 52.8  & 49.1 & 33.5 & 33.1 & 43.4 & +12.4 \\
\rowcolor{green!10}\textbf{Ours} & \textbf{Proposed} & \textbf{42.6}& \textbf{34.2}&\textbf{20.5}& \textit{23.1}&\textbf{38.7}& \textit{22.2}&\textbf{17.3}& \textit{18.8}&\textit{18.0}& \textbf{24.1}&\textbf{12.7}& \textit{24.4}&\textbf{28.2}& \textbf{32.7}&\textit{29.0}& \textbf{25.8}&+\textbf{9.6}\\
\midrule
\multicolumn{19}{c}{ImageNet \quad $\Rightarrow$ \quad ImageNet-C}\\
\midrule
Source\cite{dosovitskiy2020image} & ICLR2021 &53.0&51.8&52.1&68.5&78.8&58.5&63.3&49.9&54.2&57.7&26.4&91.4&57.5&38.0&36.2&55.8&0.0\\
TENT\cite{DequanWangetal2021}  & CVPR2021 &52.2&48.9&49.2&65.8&73&54.5&58.4&44.0&47.7&50.3&23.9&72.8&55.7&34.4&33.9&51.0&+4.8\\
CoTTA\cite{wang2022continual} & CVPR2022 &52.9&51.6&51.4&68.3&78.1&57.1&62.0&48.2&52.7&55.3&25.9&90.0&56.4&36.4&35.2&54.8&+1.0\\
VDP\cite{gan2023decorate} & AAAI2023 &52.7&51.6&50.1&58.1&70.2&56.1&58.1&42.1&46.1&45.8&23.6&70.4&54.9&34.5&36.1&50.0&+5.8\\
ViDA\cite{liu2023vida} & ICLR2024 & \textit{47.7} & \textit{42.5} & \textit{42.9} & \textbf{52.2} & \textit{56.9} & \textit{45.5} & \textit{48.9} & \textit{38.9} & \textit{42.7} & \textbf{40.7} & \textbf{24.3} & \textbf{52.8}  & \textit{49.1} & \textit{33.5} & \textit{33.1} & \textit{43.4} & +\textit{12.4} \\
% Mask-CTTA\cite{liu2023adaptive} & CVPR2024 & 47.7 & 42.5 & 42.9 & 52.2 & 56.9 & 45.5 & 48.9 & 38.9 & 42.7 & 40.7 & 24.3 & 52.8  & 49.1 & 33.5 & 33.1 & 43.4 & +12.4 \\
\rowcolor{green!10}\textbf{Ours} & \textbf{Proposed} & \cellcolor{green!10}\textbf{43.1}&\cellcolor{green!10}\textbf{38.4}&\cellcolor{green!10}\textbf{36.8}&\cellcolor{green!10}\textit{54.7}&\cellcolor{green!10}\textbf{52.2}&\cellcolor{green!10}\textbf{41.2}&\cellcolor{green!10}\textbf{48.3}&\cellcolor{green!10}\textbf{37.7}&\cellcolor{green!10}\textbf{35.6}&\cellcolor{green!10}\textit{41.1}&\cellcolor{green!10}\textit{25.2}&\cellcolor{green!10}\textit{63.5}&\cellcolor{green!10}\textbf{34.7}&\cellcolor{green!10}\textbf{27.7}&\cellcolor{green!10}\textbf{28.3}&\cellcolor{green!10}\textbf{40.5}& +\textbf{15.3}\\
\bottomrule
\end{tabular}
\end{adjustbox}
% \vspace{-0.2cm}
\end{table*}

\subsection{Quantitative analysis}
\textbf{The effectiveness on classification CTTA} validate the effectiveness of our method, we conduct experiments on \textbf{CIFAR10-to-CIFAR10-C}, \textbf{CIFAR100-to-CIFAR100-C}, and \textbf{ImageNet-to-ImageNet-C}, which consists of fifteen corruption types that occur sequentially during the test time, in \cref{tab:imagenet}. For MoASE, the average classification error is up to 55.8\% when we directly test the source model on target domains with ImageNet-C. Our method can outperform all previous methods, achieving a 15.3\% and 2.9\% improvement over the source model and previous SOTA method, respectively. Moreover, our method showcases remarkable performance across the majority of corruption types, highlighting its effective mitigation of error accumulation and catastrophic forgetting.

\textbf{The effectiveness on segmentation CTTA} As presented in \cref{tab:ACDC}, we observed a gradual decrease in the mIoUs of TENT and DePT over time, indicating the occurrence of catastrophic forgetting. In contrast, our method has a continual improvement of average mIoU (61.8→62.3→62.3) when the same sequence of target domains is repeated. Significantly, the proposed method surpasses the previous SOTA CTTA method \cite{wang2022continual} by achieving a 4.0\% increase in mIoU. This notable improvement showcases our method's ability to adapt continuously to dynamic target domains in segmentation.

\textbf{Adaptation across various model backbone.} 
We evaluate the flexibility of our MoASE with Segformer-B0\cite{xie2021segformer} and introduce the foundation model SAM\cite{kirillov2023segment}
as the pre-trained model and adapt them to continual target domains in the Cityscapes-to-ACDC CTTA scenario. Our method significantly enhanced performance in dynamic target domains, as shown in \cref{tab:CIFAR10-to-CIFAR10c-dino}, achieving improvements of 0.6\% and 0.4\% for Segformer-B0 and SAM-SETR, respectively. These findings confirm that the MoASE supports effective transfer learning across model sizes and is well-suited for a variety of real-world applications, including those in resource-limited environments like autonomous driving. Note that, we only use the pre-trained encoder of SAM\cite{kirillov2023segment} loaded into SETR model\cite{zheng2021rethinking} and add
a classification head, which is fine-tuned on the source domain.

\textbf{Exploration on domain generalization.} 
To evaluate the domain generalization (DG) capabilities of our method, we employed a leave-one-domain-out approach \cite{zhou2021domain, li2017deeper}, training on 10 of the 15 ImageNet-C domains and using the remaining 5 as unsupervised target domains. Our method adapts a pre-trained model to these 10 domains, then tests directly on the five unseen domains, as shown in \cref{tab:dg}. Impressively, it reduces average error in these domains by 12.4\%, outperforming other approaches such as ViDA by over 4.8\%. These results confirm our method's effectiveness in enhancing DG capabilities. More experiment results are in \cref{ap:ml,ap:dg,ap:five,ap:dimension,ap:comp,ap:qu}.

\begin{table*}[t]
% \vspace{-2cm}
\caption{\label{tab:ACDC} \textbf{Performance comparison for Cityscape-to-ACDC CTTA.} We sequentially repeat the same sequence of target domains three times. C-MAE stands for Continual-MAE. Mean is the average score of mIoU.}
\centering
% \vspace{-0.3cm}
\setlength\tabcolsep{1pt}%调列距
\begin{adjustbox}{width=1\linewidth,center=\linewidth}
\begin{tabular}{c|c|ccccc|ccccc|ccccc|c|c }
\toprule

\multicolumn{2}{c|}{Time}     & \multicolumn{15}{c}{$t$ \makebox[10cm]{\rightarrowfill} }                                                                              \\ \midrule
\multicolumn{2}{c|}{Round}          & \multicolumn{5}{c|}{1}    & \multicolumn{5}{c|}{2}     & \multicolumn{5}{c|}{3}  & \multirow{2}{*}{Mean$\uparrow$}   & \multirow{2}{*}{Gain$\uparrow$}  \\ \cmidrule{1-17}
Method & Venue & Fog & Night & Rain & Snow & Mean$\uparrow$ & Fog & Night & Rain & Snow  & Mean$\uparrow$ & Fog & Night & Rain & Snow & Mean$\uparrow$ & \\ \midrule
Source\cite{xie2021segformer}   &   NIPS2021 &69.1&40.3&59.7&57.8&56.7&69.1&40.3&59.7& 	57.8&56.7&69.1&40.3&59.7& 57.8&56.7&56.7&0.0\\
% TENT\cite{DequanWangetal2021}  & CVPR2021 &69.0&40.2&60.1&57.3&56.7&68.3&39.0&60.1& 	56.3&55.9&67.5&37.8&59.6&55.0&55.0&55.7&-1.0\\ 
CoTTA\cite{wang2022continual}  & CVPR2022  &70.9&41.2&62.4&59.7&58.6&70.9&41.1&62.6& 	59.7&58.6&70.9&41.0&62.7&59.7&58.6&58.6&+1.9\\ 
VDP\cite{gan2023decorate}  & AAAI2023  &70.5&41.1&62.1&59.5&  58.3    &70.4&41.1&62.2&59.4& 58.2     & 70.4&41.0&62.2&59.4& 58.2   &  58.2 & +1.5\\
DAT\cite{ni2023distribution}   & ICRA2024
&71.7 &44.4 &65.4 &62.9 &61.1 &71.6 &45.2 &63.7 &63.3 &61.0 &70.6 &44.2 &63.0 &62.8 &60.2 &60.8 &+4.1\\
SVDP\cite{yang2023exploring}  & AAAI2024
&72.1 &44.0 &65.2 &63.0 &61.1 &72.2 &44.5 &65.9 &63.5 &61.5 &72.1 &44.2 &65.6 &63.6 &61.4 &61.1 &+4.4\\
ViDA\cite{liu2023vida} & ICLR2024 & 71.6& 43.2& 66.0& 63.4& 61.1& 
  73.2& 44.5& 67.0& \textbf{63.9}& 62.2 
 & 73.2& \textit{44.6}& 67.2& 64.2& 62.3      & 61.9 
 & +5.2\\
C-MAE\cite{liu2023adaptive} & CVPR2024 & 71.9 &\textbf{44.6} &\textbf{67.4} &63.2 &61.8 &71.7 &44.9 &66.5 &63.1 &61.6 &72.3 &\textbf{45.4} &67.1 &63.1 &62.0 &61.8 &+5.1\\
 \rowcolor{green!10}\textbf{Ours} & \textbf{Proposed} & \textbf{72.4}& \textit{44.5}& \textit{66.4}& \textbf{63.8}& \textbf{61.8}& 
  \textbf{73.0}& \textbf{45.1}& \textbf{67.5}& \textit{63.5}& \textbf{62.3} 
 & \textbf{73.5}& 44.5& \textbf{67.4}& \textbf{63.5}& \textbf{62.3}      & \textbf{62.2} 
 & +\textbf{5.5}\\
 \bottomrule
 % \midrule
\end{tabular}
\end{adjustbox}
\vspace{-0.2cm}
\label{tab:CTTA}
\end{table*}

\begin{table}[t]
\begin{center}
    \begin{minipage}{0.48\textwidth}
        \centering
        \setlength\tabcolsep{2pt}%调列距
        \caption{\label{tab:CIFAR10-to-CIFAR10c-dino} \textbf{Different size of model backbone and foundation model}. mIoU score for Cityscape-to-ACDC. Mean(\%) is the average results in the first round.}
        % \vspace{-0.3cm}
        \footnotesize
        % \setlength\tabcolsep{1pt}
        % \begin{adjustbox}{width=1\linewidth,center=\linewidth}
        \resizebox{1\columnwidth}{!}{%
        \begin{tabular}{c|c|ccccc}
        \toprule
            Backbone&Method &Fog & Night & Rain & Snow & Mean$\uparrow$\\\midrule
            % \multicolumn{5}{l}{DINOv2:}\\\midrule
          \multirow{2}{*}{Segformer-B0\cite{xie2021segformer}} &ViDA\cite{liu2023vida}& 57.9 & 27.8 & 53.1 & 51.6 &  47.6\\
            
            &\cellcolor{green!10}Ours &\cellcolor{green!10}\textbf{58.2} &\cellcolor{green!10}\textbf{28.7} & \cellcolor{green!10}\textbf{53.6} &\cellcolor{green!10}\textbf{52.2} & \cellcolor{green!10}\textbf{48.2}\\
            \midrule
            % \multicolumn{5}{l}{SAM:}\\\midrule
          \multirow{2}{*}{SAM\cite{kirillov2023segment}-SETR \cite{zheng2021rethinking}} &ViDA\cite{liu2023vida}& 76.5& 47.2 &68.1 &70.7 &65.6\\
            
            &\cellcolor{green!10}Ours &\cellcolor{green!10}\textbf{76.8}&\cellcolor{green!10}\textbf{47.6} &\cellcolor{green!10}\textbf{68.7} &\cellcolor{green!10}\textbf{71.0}  & \cellcolor{green!10}\textbf{66.0}\\
            \bottomrule
        \end{tabular}
        }
        % \end{adjustbox}
    \end{minipage}
        \hspace{0.02\textwidth}
 \begin{minipage}{0.48\textwidth}
    \centering
        \setlength\tabcolsep{3pt}%调列距
        \caption{\label{tab:dg} \textbf{The domain generalization comparisons on ImageNet-C.} Results are evaluated on ViT-base. Mean(\%) represents the performance on unseen target domains.}
        \footnotesize
        % \vspace{-0.3cm}
        % \setlength\tabcolsep{1pt}
        % \begin{adjustbox}{width=1\linewidth,center=\linewidth}
        \resizebox{0.99\columnwidth}{!}{%
        \begin{tabular}{c|ccccc|c}
        \toprule
        Method & bri. & contrast & elastic & pixelate & jpeg 
        & Mean$\downarrow$ \\
        \midrule
        Source\cite{xie2021segformer}&26.4&91.4&57.5&38.0&36.2&49.9\\
       CoTTA\cite{wang2022continual} &25.3&88.1&55.7&36.4&34.6&48.0\\
       ViDA\cite{liu2023vida} & 24.6 & 68.2 & 49.8 & 34.7 & 34.1 & 42.3 \\ 
       \rowcolor{green!10} \textbf{Ours} & \textbf{25.4}& \textbf{65.5}& \textbf{37.3}& \textbf{29.5}& \textbf{29.6}& \textbf{37.5}\\ 
        \bottomrule
        \end{tabular}
        }
        % \end{adjustbox}
            \end{minipage}
\end{center}
\vspace{-0.5cm}
\end{table}

\vspace{-0.3cm}
\subsection{Ablation study}
\vspace{-0.1cm}
\textbf{Different number of experts.}
In this study, we evaluate the impact of varying expert module counts in the MoASE on mIoU scores under adverse weather conditions within the Cityscape-to-ACDC CTTA scenario, as shown in \cref{tab:experts}. We tested configurations with 2 to 16 expert modules. The baseline, with no experts, achieved a mean mIoU of 58.6\%. Our findings reveal that a setup with four experts delivered the highest scores in Fog and Snow, achieving an overall mean mIoU of 61.8\%. These results demonstrate that the relationship between the number of experts and model performance in CTTA is not linear; rather, precise activation decomposition is essential for optimal performance.

\textbf{Effectiveness of each proposed module.} 
We conducted an ablation study in the Cityscape-to-ACDC CTTA scenario to evaluate the effects of our modules including Spatial Differentiate Dropout (SDD), Domain Agnostic Gate (DAG), Activation Sparsity Gate (ASG), and Homeostatic-Proximal (HP) loss. \cref{tab:ablation} shows $Ex_{0}$ as the baseline using the CoTTA \cite{wang2022continual} and $Ex_{1}$ adding a 4-expert MoE architecture to $Ex_{0}$. However, simply adding MoE did not enhance performance; it instead decreased by 0.7\%. In contrast, implementing our SDD in $Ex_{2}$ improved segmentation results by 3.2\%. Further introductions of our modules from $Ex_{3}$ to $Ex_{5}$ increased mIoU from 61.5\% to 62.2\%, confirming the effectiveness of our proposed methods.

\begin{figure*}[t]
% \vspace{-0.2cm}
\centering
\includegraphics[width=0.99\linewidth]{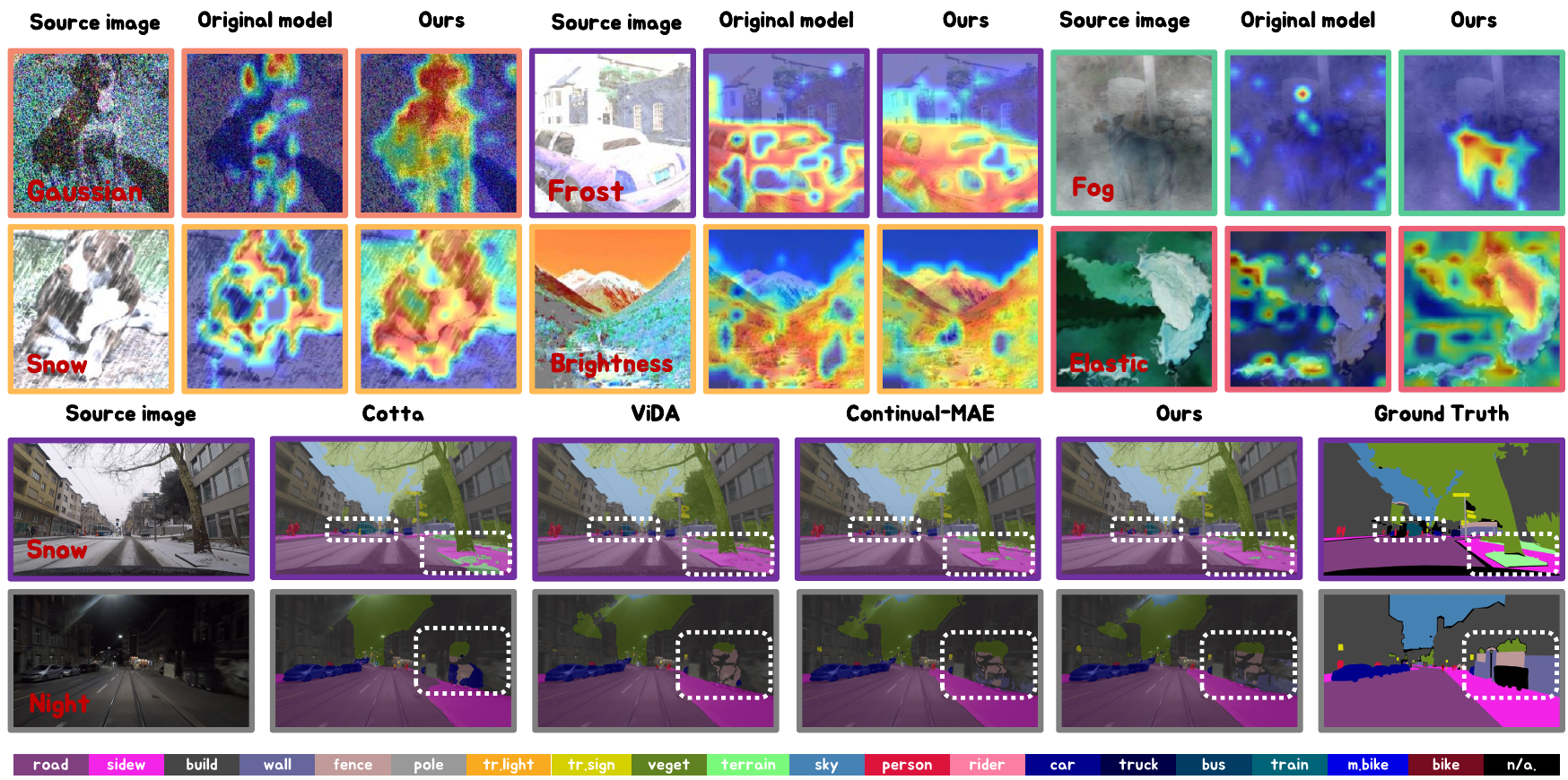}
% \vspace{-0.2cm}
\caption{\textbf{The qualitative analysis} of the CAM and the segmentation qualitative comparison of our method with previous SOTA methods on the ACDC dataset. }
% Our method could better segment different pixel-wise classes such as shown in the white box.}
\label{fig:vis}
% \vspace{-0.3cm}
\end{figure*}

\begin{table}[t]
\begin{center}
      \begin{minipage}{0.48\textwidth}
        \centering
        \setlength\tabcolsep{7pt}%调列距
        \caption{\label{tab:experts} \textbf{Different number of experts.} mIoU score for Cityscape to ACDC with MoASE within the first round.}
        % \vspace{-0.3cm}
        \footnotesize
        % \begin{adjustbox}{width=1\linewidth,center=\linewidth}
                \resizebox{0.99\columnwidth}{!}{%
        \begin{tabular}{c|cccc|c}
        \toprule
        num. E. & Fog & Night & Rain & Snow 
        & Mean$\uparrow$ \\
        \midrule
        $E=0$&70.9&41.2&62.4&59.7&58.6\\
        \midrule
        $E=2$&71.6&44.0&\textbf{66.5}&63.7&61.5\\
       \rowcolor{green!10} \textbf{$E=4$} & \textbf{72.4}& \textbf{44.5}& 66.4& \textbf{63.8}& \textbf{61.8}\\
       $E=8$ & 71.4 & 44.0 & 65.0 & 61.7 & 60.5 \\ 
        $E=16$ & 71.5 & 44.1 & 65.9 & 63.3& 61.2 \\ 
        \bottomrule
        \end{tabular}
        }
        % \end{adjustbox}
    \end{minipage}
    \hspace{0.02\textwidth}
    \begin{minipage}{0.48\textwidth}
     \centering
        \setlength\tabcolsep{5pt}%调列距
        \caption{\label{tab:ablation} \textbf{Ablation study} on Cityscape to ACDC showing the effectiveness of each module.}
        \footnotesize
        % \vspace{-0.3cm}
        % \setlength\tabcolsep{1pt}
        % \begin{adjustbox}{width=1\linewidth,center=\linewidth}
        \resizebox{0.99\columnwidth}{!}{%
        \begin{tabular}{c|ccccc|c}
        \toprule
         & MoE & SDD & DAG & ASG & HP & Mean$\uparrow
         $ \\\midrule
        $Ex_{0}$ & - & - & - & -& - &58.6 \\ 
        $Ex_{1}$ & \checkmark & - & - & -& - &57.9 \\ 
        $Ex_{2}$ & \checkmark & \checkmark & - & -& - &61.1\\
        $Ex_{3}$ & \checkmark & \checkmark & \checkmark & -& -&61.5\\
        $Ex_{4}$ & \checkmark & \checkmark & \checkmark &\checkmark& - & 62.0\\
        \cellcolor{green!10}$Ex_{5}$  & 
        \cellcolor{green!10}\checkmark &
        \cellcolor{green!10}\checkmark & \cellcolor{green!10}\checkmark & \cellcolor{green!10}\checkmark & \cellcolor{green!10}\checkmark & \cellcolor{green!10}\textbf{62.2}\\
        \bottomrule
        \end{tabular}
        }
    \end{minipage}
\end{center}
\vspace{-0.5cm}
\end{table}

\subsection{Qualitative analysis}
\textbf{CAM visualization.} We conducted a qualitative analysis of the CAM on ImageNet-C, as illustrated at the top of \cref{fig:vis}. MoASE effectively concentrates on regions relevant to the target categories, such as dogs and cars, during classification decisions. In contrast, the original model's attention is dispersed due to continuous domain shifts. These findings underscore MoASE's superiority.

\textbf{Segmentation results.} For further validation, we provide additional qualitative comparisons in the Cityscapes-to-ACDC CTTA scenario, shown at the bottom of \cref{fig:vis}. Our method outperforms CoTTA\cite{wang2022continual}, ViDA\cite{liu2023vida}, and Continual-MAE\cite{liu2023adaptive} in producing segmentation maps for snow and night domains, precisely differentiating sidewalks from roads and identifying small objects like people, vegetation, and fences (highlighted in the white box). These results highlight our method's precise segmentation capabilities and robustness against dynamic domain shifts. Additionally, our segmentation maps align closely with the Ground Truth, leading to significant visual improvements.

\section{Conclusion and limitations}
\label{conclusion}
Inspired by the human visual system where specialized cone cells perceive different light spectrum sections, this paper explores deep neural networks (DNNs) under the Continual Test-time Adaptation (CTTA) scenario. We introduce a novel architecture, the Mixture-of-Activation-Sparsity-Experts (MoASE). This design decouples neural activation into high-activation and low-activation components using Spatial Differentiate Dropout (SDD) and enhances domain-specific feature extraction while improving the perception of domain-agnostic objects through expert modules. Enhanced by a multi-gate module and Homeostatic-Proximal (HP) loss, MoASE surpasses state-of-the-art baselines in four benchmarks. However, the additional computational overhead introduced to the system can potentially be further minimized in future research. Exploring methods such as feature modulation or quantization could enhance model efficiency.

% \clearpage

\bibliography{neurips_2024}
\bibliographystyle{unsrt}

\clearpage

\appendix
\section*{Appendix}
The supplementary materials accompanying this paper\footnote{Part of the figure downloaded from \url{https://www.wikiwand.com},
\url{https://reactome.org},
\url{https://www.janssenwithme.com}
} provide a comprehensive quantitative and qualitative analysis of the proposed method. To begin with, we provide the error bound analysis in \cref{ap:bound}. In Section \cref{ap:details}, we offer detailed insights into the experimental settings, including datasets and additional implementation specifics. Furthermore, we expand our investigation into the effects of varying middle-layer dimensions on CIFAR10-C performance in Section \cref{ap:ml}. 
To evaluate the domain generalization capabilities of our method, we conducted experiments that directly tested its performance across a varying number of unseen domains, as detailed in Section \cref{ap:dg}. Section \cref{ap:five} describes five rounds of semantic segmentation CTTA experiments, providing deeper insights into the method's adaptability and effectiveness.
Additionally, we include an experiment in \cref{ap:dimension} that elucidates our decision to apply Spatial Differentiate Dropout (SDD) on the spatial-wise token dimension rather than the channel dimension, highlighting the strategic rationale behind this choice. Finally, Section \cref{ap:qu} presents further qualitative analysis, offering a visual and descriptive exploration of the method's performance enhancements and capabilities.

\section{A bound relating the source and target error}
\label{ap:bound}
We now turn our attention to establishing bounds on the target domain generalization performance of a discriminator trained within the source domain. Initially, we aim to quantify the target error about the source error\cite{ben2006analysis, ben2010theory}, the discrepancy between the labeling functions $f_S$ and $f_T$, and the divergence between the distributions $D_S$ and $D_T$. Given the practical assumption that the difference between labeling functions is minimal, our analysis primarily concentrates on quantifying distribution divergence. Particularly, we explore methods for estimating this divergence using finite samples of unlabeled data from $D_S$ and $D_T$. A pertinent measure for evaluating divergence between distributions is the variation divergence:
\begin{equation}
d(D, D') = 2 \mathop{\mathrm{sup}}_{W \in \mathcal{W}} | \mathrm{Pr}_{D}[W] - \mathrm{Pr}_{D'}[W] |
\label{apeq:vd}
\end{equation}
where $\mathcal{W}$ represents the set of measurable subsets with respect to distributions $D$ and $D'$. We employ this measure to establish an initial bound on the target error of a discriminator.

\begin{theorem}
For a hypothesis h,
\begin{equation}
\epsilon_{T}(h) \leq \epsilon_{S}(h) + d(D, D') + min\{E_{D}[|f_{S}(\textbf{x})-f_{T}(\textbf{x})|], E_{D'}[|f_{S}(\textbf{x})-f_{T}(\textbf{x})|] \}
\end{equation}
\end{theorem}

\begin{proof}
Given that $\epsilon_{T}(h) = \epsilon_{T}(h, f_{T})$ and $\epsilon_{S}(h) = \epsilon_{S}(h, f_{S})$, let $\phi_{S}$ and $\phi_{T}$ denote the density functions of distributions $D$ and $D'$, respectively.
\begin{align}
\epsilon_{T}(h) &= \epsilon_{T}(h) + \epsilon_{S}(h) - \epsilon_{S}(h) + \epsilon_{S}(h,f_{T}) - \epsilon_{S}(h,f_{T}) \\ 
&\leq \epsilon_{S}(h) + |\epsilon_{S}(h,f_{T}) - \epsilon_{S}(h,f_{S})| + |\epsilon_{T}(h,f_{T}) - \epsilon_{S}(h,f_{T})| \\ 
&\leq \epsilon_{S}(h) + E_{D}[|f_{S}(\textbf{x})-f_{T}(\textbf{x})|] + |\epsilon_{T}(h,f_{T}) - \epsilon_{S}(h,f_{T})| \\ 
&\leq \epsilon_{S}(h) + E_{D}[|f_{S}(\textbf{x})-f_{T}(\textbf{x})|] + \int|\phi_{S}(\textbf{x}) - \phi_{T}(\textbf{x})||h(\textbf{x})-f_{T}(\textbf{x})|d\textbf{x} \\
&\leq \epsilon_{S}(h) + E_{D}[|f_{S}(\textbf{x})-f_{T}(\textbf{x})|] + d(D, D')
\end{align}
\end{proof}

In the first line, we can alternatively add and subtract $\epsilon_{T}(h,f_{S})$ instead of $\epsilon_{T}(h,f_{T})$. This adjustment leads to the same theoretical bound, albeit with the expectation calculated concerning distribution $D'$ rather than $D$. Choosing the smaller of the two bounds provides a more favorable constraint.

However, approximating the error of the optimal hyperplane discriminator for arbitrary distributions is recognized as an NP-hard problem. In response, we approximate the optimal hyperplane discriminator by minimizing a convex upper bound on the error, a common approach in classification. It is crucial to recognize that this method does not yield a valid upper bound on the target error. Therefore, we utilize the $Jensen$-$Shannon\ (JS)\ divergence$ between two adjacent domains as an approximation of $\mathcal{H}$-divergence with respect to $d(D, D')$.

\section{Detailed experiment settings}
\label{ap:details}

Following \cite{wang2022continual}, in classification CTTA tasks, the model's online prediction capabilities are evaluated immediately upon processing the input data. To mimic real-life scenarios of environmental change, we cyclically expose the model to these conditions in a repeated sequence (Fog→Night→Rain→Snow) across multiple cycles.

\textbf{Dataset.} We evaluate our method on three classification CTTA benchmarks: CIFAR10-to-CIFAR10C, CIFAR100-to-CIFAR100C \cite{krizhevsky2009learning}, and ImageNet-to-ImageNet-C \cite{hendrycks2019benchmarking}. CIFAR10-C, CIFAR100-C, and ImageNet-C are specifically designed to test the robustness of machine learning models against common real-world image corruptions and perturbations, including noise, blur, and compression. For the segmentation CTTA context \cite{yang2023exploring,liu2023vida}, we conduct evaluations using the Cityscapes-to-ACDC benchmark. Cityscapes dataset \cite{cordts2016cityscapes} is used as the source domain, while the ACDC dataset \cite{sakaridis2021acdc} serves as the target domain, assessing adaptation effectiveness.

\textbf{Implementation Details.} In our CTTA experiments, we meticulously adhere to the implementation protocols established in previous research \cite{wang2022continual} to ensure both consistency and comparability across our studies. As part of our augmentation strategy, we employ a range of image resolution scaling factors [0.5, 0.75, 1.0, 1.25, 1.5, 1.75, 2.0] to generate inputs for the teacher model, as suggested by Wang et al. \cite{wang2022continual}. The batch size is set to 40 according to ViDA\cite{liu2023vida}. For the segmentation task, the learning rate decay is set at 0.5 for each 3200 iterations. The hyperparameter $\mu$ in \cref{apeq:hploss} is set to 1, aligning with the findings from previous studies \cite{li2020federated}. The threshold is manually set for SDD as $q\%=\{\frac{id_{h}}{E},...,\frac{1}{2},\frac{id_{l}}{E},...,\frac{1}{2}\}$, where $id_{h}$ and $id_{l}$ suggests the expert ID for domain-agnostic experts and domain-specific experts. In the initialization process of the MoASE, we adopt the methodology from Adaptformer \cite{chen2022adaptformer}. The weights of the down-projection layers are initialized using Kaiming Normal initialization \cite{he2015delving}, which is well-suited for preserving the mean and variance of inputs through the layers during forward and backward passes. Conversely, the biases in the additional networks and the weights of the up-projection layers are set to zero. This combination of initialization is designed to optimize the performance and stability of the model during training, ensuring that each component contributes effectively to the overall learning process. The random seed is set to 1.
\vspace{-0.1cm}

\begin{table*}[htb]
\caption{\textbf{Classification error rate(\%) for CIFAR10-to-CIFAR10-C online CTTA task.} Gain(\%) represents the percentage of improvement in model accuracy compared with the source method with different hidden size $h$ in Activation Sparsity Experts.}
% \vspace{-0.2cm}
\centering
\setlength\tabcolsep{4pt}%调列距
\begin{adjustbox}{width=1\linewidth,center=\linewidth}
\begin{tabular}{c|c|ccccccccccccccc|c}
\toprule
% Time & \multicolumn{15}{l|}{$t\xrightarrow{\hspace*{13.5cm}}$}& \\ \hline
 &Method  &
 \rotatebox[origin=c]{70}{Gaussian} &  \rotatebox[origin=c]{70}{shot} &  \rotatebox[origin=c]{70}{impulse} &  \rotatebox[origin=c]{70}{defocus} &  \rotatebox[origin=c]{70}{glass} &  \rotatebox[origin=c]{70}{motion} &  \rotatebox[origin=c]{70}{zoom} &  \rotatebox[origin=c]{70}{snow} &  \rotatebox[origin=c]{70}{frost} &  \rotatebox[origin=c]{70}{fog}  &  \rotatebox[origin=c]{70}{brightness} &  \rotatebox[origin=c]{70}{contrast} &  \rotatebox[origin=c]{70}{elastic} &  \rotatebox[origin=c]{70}{pixelate} &  \rotatebox[origin=c]{70}{jpeg}
& Mean$\downarrow$ \\\midrule
$Ex_1$&Source&60.1&53.2&38.3&19.9&35.5&22.6&18.6&12.1&12.7&22.8&5.3&49.7&23.6&24.7&23.1&28.2\\
$Ex_2$&CoTTA &58.7&51.3&33.0&20.1&34.8&20&15.2&11.1&11.3&18.5&4.0&34.7&18.8&19.0&17.9&24.6\\
\midrule
$Ex_3$& $h=24$ &65.4&56.3&32.8&\textit{16.2}&31.4&\textit{13.5}&\textit{9.4}&8.8&8.3&10.6&\textit{3.0}&16.1&\textit{12.1}&\textit{15.5}&\textit{13.3}&20.8\\
\cellcolor{green!10}$Ex_4$& \cellcolor{green!10}$h=48$ &\cellcolor{green!10}\textbf{43.7}& \cellcolor{green!10}\textbf{31.3}&\cellcolor{green!10}\textbf{25.1}&\cellcolor{green!10}16.5& \cellcolor{green!10}\textbf{28.1}&\cellcolor{green!10}13.8&\cellcolor{green!10}9.7&\cellcolor{green!10}\textbf{8.3}& \cellcolor{green!10}\textbf{7.1}& \cellcolor{green!10}\textit{10.1}&\cellcolor{green!10}\textbf{3.0}&\cellcolor{green!10}\textbf{12.9}& \cellcolor{green!10}\textbf{12.0}&\cellcolor{green!10}\textit{16.3}& \cellcolor{green!10}13.5&\cellcolor{green!10}\textbf{16.8}\\
$Ex_5$& $h=96$ &63.4&52.4&36.4&18.7&31.9&15.4&11.2& 9.7&8.5&12.0&3.3&15.5&12.8&17.1&14.2&21.5\\
$Ex_6$& $h=192$ &\textit{59.1}&\textit{47.7}&\textit{27.4}&\textbf{
16.0}&\textit{28.6}&\textbf{13.0}&\textbf{9.2}&\textit{8.4}&\textit{7.7}&\textbf{9.8}&\textit{3.0}&\textit{14.5}&12.3&\textit{15.5}&\textit{13.2}&\textit{19.0}\\
\bottomrule
\end{tabular}
\end{adjustbox}
\vspace{-0.2cm}
\label{tab:hidden}
\end{table*}

\vspace{-0.2cm}
\section{How does the middle-layer dimension influence the performance?}
\label{ap:ml}

According to the results presented in \cref{tab:hidden}, the table provides a comparative analysis of the baseline method's performance against three experimental configurations with hidden sizes of 24, 48, 96, and 192 on the CIFAR10-to-CIFAR10-C online CTTA task. An interesting observation from the table is that the correlation between model performance and the size of the hidden dimension does not follow a linear trend. Specifically, a hidden size of $h=48$ demonstrates optimal model performance with a mean error rate of 16.8\%, whereas a smaller hidden size of $h=24$ results in a higher error rate of 20.8\% where the larger hidden size of $h=96$ with an even higher error rate of 21.5\%. Moreover, further increasing the hidden dimension to $h=192$ does not yield a significant improvement in performance, with the error rate marginally reduced to 19.0\%, but it entails a substantial increase in the model size.

The nonlinear relationship between model architecture and task complexity is critical in resource-limited settings like mobile computing or autonomous driving, where efficient yet effective models are essential. Therefore, we selected a hidden size of $h=48$ for our experiments to best balance performance and computational efficiency.

\section{Domain generalization on a different number of unseen target domains}
\label{ap:dg}

\begin{table*}[t]
\centering
 \caption{\label{tab:dg10} \textbf{The domain generalization experiments on ImageNet-C}, where the source model was continually adapted on the first 5 domains and directly tested on 10 unseen domains. The evaluation of the results was conducted using ViT-base.}
% \vspace{-0.2cm}
\small
\setlength\tabcolsep{8pt}
% \begin{adjustbox}{width=1\linewidth,center=\linewidth}
\resizebox{0.99\columnwidth}{!}{%
\begin{tabular}{l|cccccccccc|c}
\toprule
 &  \multicolumn{10}{c|}{\textbf{Directly test on 10 unseen domains}}& \multicolumn{1}{c}{\textbf{Unseen}} \\ \midrule
Method & \rotatebox[origin=c]{70}{motion} & \rotatebox[origin=c]{70}{zoom} & \rotatebox[origin=c]{70}{snow} & \rotatebox[origin=c]{70}{frost} & \rotatebox[origin=c]{70}{fog} & \rotatebox[origin=c]{70}{brightness} & \rotatebox[origin=c]{70}{contrast} & \rotatebox[origin=c]{70}{elastic\_trans} & \rotatebox[origin=c]{70}{pixelate} & \rotatebox[origin=c]{70}{jpeg}
& Mean$\downarrow$ \\\midrule
Source\cite{xie2021segformer}& 58.5&63.3&49.9&54.2&57.7&26.4&91.4&57.5&38.0&36.2&53.3\\
TENT\cite{DequanWangetal2021}& 56.0 & 61.3& 45.7 &49.6   &  56.6 &24.8&94.0&55.6&37.1&35.1&51.6\\
CoTTA\cite{wang2022continual}& 57.3 & 62.1& 49.1 &52.0   &  57.1 &26.4&91.9&57.1&37.6&35.3&52.6\\
ViDA\cite{liu2023vida} &46.4 &\textbf{52.7}&39.8  & 43.7  &\textbf{42.2}   &\textbf{23.5}&71.5&49.6&33.9&33.3&43.7\\ 
\cellcolor{green!10}\textbf{Ours} & \cellcolor{green!10}\textbf{43.9} &\cellcolor{green!10}\textit{53.7} &\cellcolor{green!10}\textbf{38.0}  & \cellcolor{green!10}\textbf{37.6}  &\cellcolor{green!10}\textit{46.0}   &\cellcolor{green!10}\textit{24.0}&\cellcolor{green!10}\textbf{66.0}&\cellcolor{green!10}\textbf{42.8}&\cellcolor{green!10}\textbf{28.9}&\cellcolor{green!10}\textbf{29.6}&\cellcolor{green!10}\textbf{41.0}\\ 
\bottomrule
\end{tabular}
}
% \end{adjustbox}
% \vspace{-0.2cm}
\end{table*}

In this study, we adhere to the leave-one-domain-out principle, similar to previous works \cite{zhou2021domain, li2017deeper}. We select a subset of ImageNet-C domains as new source domains for model updating while designating the remaining domains as target domains where no adaptation occurs. This approach diverges from earlier domain generalization experiments; here, we implement an unsupervised continual test-time adaptation (CTTA) approach, updating the model solely on these unlabeled source domains. The initial model weights are derived exclusively from ImageNet pre-trained parameters.

We specify that 5 out of 15 domains from ImageNet-C are used as source domains, with the other 10 and 8 out of 15 domains serving as unseen target domains. The results, as detailed in \cref{tab:dg10}, are noteworthy. Our method achieves a significant reduction of 12.3\% in average error across these unseen domains. These encouraging results validate the domain generalization (DG) capabilities of our approach, demonstrating its effectiveness in extracting domain-shared knowledge. This success offers a novel perspective on enhancing DG performance within an unsupervised framework.

\begin{table*}[htb]
% \vspace{-2cm}
\caption{\label{aptab:ACDC} \textbf{5 rounds segmentation CTTA on Cityscape-to-ACDC.} We sequentially repeat the same sequence of target domains 5 times with different number of experts in our proposed MoASE. Mean is the average score of mIoU.}
\centering
% \vspace{-0.15cm}
\setlength\tabcolsep{0.1pt}%调列距
\begin{adjustbox}{width=1\linewidth,center=\linewidth}
\begin{tabular}{c|ccccc|ccccc|ccccc|ccccc|ccccc|c }
\midrule

% \multicolumn{2}{c|}{Time}     & \multicolumn{15}{c}{$t$ \makebox[10cm]{\rightarrowfill} }                                                        
Round       & \multicolumn{5}{c|}{1}    & \multicolumn{5}{c|}{2}     & \multicolumn{5}{c|}{3} & \multicolumn{5}{c|}{4} & \multicolumn{5}{c|}{5}  & \multirow{2}{*}{Mean}   \\ 
\cmidrule{1-26}
Method  & Fog & Night & Rain & Snow & Mean & Fog & Night & Rain & Snow  & Mean & Fog & Night & Rain & Snow & Mean &   Fog & Night & Rain & Snow & Mean &  Fog & Night & Rain & Snow & Mean \\ \midrule
Source   &69.1&40.3&59.7&57.8&56.7&69.1&40.3&59.7& 	57.8&56.7&69.1&40.3&59.7& 57.8&56.7&56.7 &40.3&59.7& 57.8&56.7&56.7 &40.3&59.7& 57.8&56.7&56.7 \\ 
CoTTA  &70.9&41.2&62.4&59.7&58.6&70.9&41.1&62.6& 	59.7&58.6&70.9&41.0&62.7&59.7&58.6&70.9&41.0&62.7&59.7&58.6&70.9&41.0&62.8&59.7&58.6&58.6\\ 
CoTTA$^{*}$  &71.4&45.1& 66.8&63.0&61.8&71.9&43.3&65.9& 	61.8&60.7&70.1&39.5&63.0&60.6&58.3&68.1&39.8&62.0&60.0&57.4&68.0&38.3&62.0&59.8&57.0&59.0\\
ViDA  & 71.6& 43.2& 66.0& 63.4& 61.1& 
  73.2& 44.5& 67.0& 63.9& 62.2 
 & 73.2& 44.6& 67.2& 64.2& 62.3  &70.9&44.0&66.0&63.2&61.0&72.0&43.7&66.3&63.1&61.3&61.6\\ 
 \midrule
 E=2  & 71.6& 44.0& 66.5& 63.8& 61.5 & 
72.7&45.6 & 67.5& 63.9& 62.3
 & 72.5& 44.6& 67.5& 63.0& 62.0    & 71.8 & 44.4 & 67.3 & 63.3 & 61.7 & 71.3 & 43.2 & 66.8 & 62.3 & 60.9&61.7 \\
\cellcolor{green!10}E=4 & 
\cellcolor{green!10}\textbf{72.4}& \cellcolor{green!10}\textbf{44.5}& \cellcolor{green!10}\textbf{66.4}& \cellcolor{green!10}\textbf{63.8}& \cellcolor{green!10}\textbf{61.8}& 
\cellcolor{green!10}\textbf{73.0}& \cellcolor{green!10}\textbf{45.1}& \cellcolor{green!10}\textbf{67.5}& \cellcolor{green!10}\textbf{63.5}& \cellcolor{green!10}\textbf{62.3}& \cellcolor{green!10}\textbf{73.5}& \cellcolor{green!10}\textbf{44.5}& \cellcolor{green!10}\textbf{67.4}& \cellcolor{green!10}\textbf{63.5}& \cellcolor{green!10}\textbf{62.3}&
\cellcolor{green!10}\textbf{71.8}& \cellcolor{green!10}\textbf{44.3}& \cellcolor{green!10}\textbf{66.3}& \cellcolor{green!10}\textbf{62.7}& \cellcolor{green!10}\textbf{61.3}&
\cellcolor{green!10}\textbf{72.0}& \cellcolor{green!10}\textbf{44.1}& \cellcolor{green!10}\textbf{66.5}& \cellcolor{green!10}\textbf{62.2}&  \cellcolor{green!10}\textbf{61.2}&
\cellcolor{green!10}\textbf{61.8}\\
 E=8 & 71.4 & 44.0 & 65.0 & 61.7 & 60.5 & 72.3 & 43.5 & 65.5 & 62.0 & 60.8 & 71.6 & 43.1 & 64.6 & 61.4 & 60.2& 71.2 & 42.5 & 64.7 & 62.4 & 60.2& 70.6 & 43.3 & 65.2 & 61.4 & 60.1&60.4  \\
 E=16 & 71.5 & 44.1 & 65.9 & 63.3 & 61.2  & 72.7 & 44.6 & 67.3 & 63.3 &62.0 & 72.2 &  43.5 & 67.8 & 63.0 &61.6 & 71.3 & 43.2 & 67.3 & 62.3 & 61.0& 70.8 & 43.0 & 66.7 & 61.8 & 60.6 &61.3 \\
\bottomrule
\end{tabular}
\end{adjustbox}
% \vspace{-0.3cm}
% \label{tab:CTTA}
\end{table*}

\section{Additional experiments with different number of experts}
\label{ap:five}

We present the segmentation CTTA experiment across 5 rounds in \cref{aptab:ACDC}, showing consistent mean mIoU enhancement during the initial rounds (1-3) and stable performance thereafter. Our method achieved a 0.2\% mIoU improvement over the previous SOTA after averaging results from all rounds. Additionally, \cref{aptab:ACDC} (CoTTA$^{*}$) details our adjustment of CoTTA's hyperparameters, specifically increasing the learning rate to 2e-4. This change initially improves performance but leads to a decline in segmentation accuracy and catastrophic forgetting in later rounds.

Furthermore, we present the comprehensive results in \cref{tab:experts} with different numbers of experts, which are detailed in the last four lines of \cref{aptab:ACDC}. These results further underscore the robustness and flexibility of our proposed Mixture of Activation Sparsity Experts (MoASE), demonstrating its capability to achieve satisfactory outcomes across various experimental settings. Moreover, we observed a similar trend between the number of experts and the dimension of the hidden layer, suggesting that the appropriate activation decomposition is more crucial than the model size.

\begin{table*}[t]
% \vspace{-0.1cm}
\caption{\label{tab:imagenet-dimension} \textbf{Classification error rate(\%) for ImageNet-to-ImageNet-C online CTTA task.} Gain(\%) represents the percentage of improvement in model accuracy compared with the source method.}
% \vspace{-0.3cm}
\centering
\setlength\tabcolsep{1pt}%调列距
\begin{adjustbox}{width=1\linewidth,center=\linewidth}
\begin{tabular}{c|c|ccccccccccccccc|cc}
\toprule
% Time & \multicolumn{15}{l|}{$t\xrightarrow{\hspace*{13.5cm}}$}& \\ \midrule
Method & Venue &
 \rotatebox[origin=c]{70}{Gaussian} & \rotatebox[origin=c]{70}{shot} & \rotatebox[origin=c]{70}{impulse} & \rotatebox[origin=c]{70}{defocus} & \rotatebox[origin=c]{70}{glass} & \rotatebox[origin=c]{70}{motion} & \rotatebox[origin=c]{70}{zoom} & \rotatebox[origin=c]{70}{snow} & \rotatebox[origin=c]{70}{frost} & \rotatebox[origin=c]{70}{fog}  & \rotatebox[origin=c]{70}{brightness} & \rotatebox[origin=c]{70}{contrast} & \rotatebox[origin=c]{70}{elastic\_trans} & \rotatebox[origin=c]{70}{pixelate} & \rotatebox[origin=c]{70}{jpeg}
& Mean$\downarrow$ & Gain$\uparrow$ \\
\midrule
\multicolumn{19}{c}{CIFAR10 \quad $\Rightarrow$ \quad CIFAR10-C}\\
\midrule
Source\cite{dosovitskiy2020image} & ICLR2021 &60.1&53.2&38.3&19.9&35.5&22.6&18.6&12.1&12.7&22.8&5.3&49.7&23.6&24.7&23.1&28.2&0.0\\
CoTTA\cite{wang2022continual} & CVPR2022 &\textit{58.7}&51.3&33.0&20.1&34.8&20&15.2&11.1&11.3&18.5&4.0&34.7&18.8&19.0&17.9&24.6&+3.6\\
Ours-Channel & Proposed & 59.2 & \textit{48.0} & \textit{28.1} & \textbf{15.8} & \textit{29.5} & \textbf{13.5} & \textbf{9.3} & \textit{8.4} & \textit{7.7} & \textbf{10.2} & \textbf{3.3} & \textbf{13.6}  & \textit{12.1} & \textit{15.6} & \textbf{13.3} & \textit{19.2} & +\textit{9.0}\\
% Mask-CTTA\cite{liu2023adaptive} & CVPR2024 & 47.7 & 42.5 & 42.9 & 52.2 & 56.9 & 45.5 & 48.9 & 38.9 & 42.7 & 40.7 & 24.3 & 52.8  & 49.1 & 33.5 & 33.1 & 43.4 & +12.4 \\
\rowcolor{green!10}\textbf{Ours-Token} & \textbf{Proposed} & \textbf{52.4}& \textbf{40.1}&\textbf{25.6}&\textit{16.2}& \textbf{28.5}&\textit{13.7}&\textit{9.5}&\textbf{8.1}& \textbf{7.3}& \textbf{10.2}&\textbf{3.1}&\textit{14.0}& \textbf{12.0}& \textbf{15.6}& \textit{13.6}& \textbf{18.0}& +\textbf{10.2}\\
\bottomrule
\end{tabular}
\end{adjustbox}
% \vspace{-0.2cm}
\end{table*}

\section{Why conduct the SDD on spatial-wise token dimension instead of channel dimension?}
\label{ap:dimension}
In our experiments, we implemented Spatial Differentiate Dropout (SDD) on both the spatial-wise token dimension $n$ and the channel dimension $d$ of the input feature $\mathcal{F}\in \mathbb{R}^{b\times n \times d}$, as shown in \cref{tab:imagenet-dimension}. We observed that in most target domains, channel-wise MoASE maintained satisfactory performance levels that were comparable to those of spatial-wise MoASE, with an average error rate of 19.2\% and a +9.0\% improvement over the baseline from the source domain. However, in the first two target domains, implementing SDD specifically on the spatial-wise token dimension significantly enhanced model performance, reducing the error rate to an average of over 10\%. We attribute this improvement to the distribution of domain-related information—such as styles and noises—across the entire images on each token. Consequently, decomposing the activation at the token dimension appears more effective for differentiating between domain-specific and domain-agnostic features, thereby tailoring the model more adeptly to the nuances of each domain.

\begin{table*}[htb]
\vspace{-0.3cm}
\caption{\label{tab:comp} \textbf{Analysis of computational costs} for our proposed MoASE and previous SOTA baselines including a detailed comparison in terms of the number of parameters, FLOPs, and memory usage.}
\centering
% \vspace{-0.3cm}
\setlength\tabcolsep{10pt}%调列距
\begin{adjustbox}{width=1\linewidth,center=\linewidth}
\begin{tabular}{c|c|cc|cc|cc}
\toprule
\multicolumn{2}{c|}{Round}          & \multicolumn{2}{c|}{Parameters (M)}    & \multicolumn{2}{c|}{FLOPs (GMac)}     & \multicolumn{2}{c}{GPU Memory (M)}  \\ \cmidrule{1-8}
Method & Venue & Number & Overhead (\%) & Number & Overhead (\%) & Number & Overhead (\%) \\ \midrule
CoTTA\cite{wang2022continual}  & CVPR2022  &173.14&-&650.72&-&20842&-\\ 
ViDA\cite{liu2023vida}  & ICLR2024   &177.95&2.8&668.25&2.7&22330&6.7\\ 
 \rowcolor{green!10}\textbf{Ours} & \textbf{Proposed}  &180.44&4.2&677.31&4.0&24530&17.6\\
 \bottomrule
 % \midrule
\end{tabular}
\end{adjustbox}
\vspace{-0.3cm}
\end{table*}

\section{Computation analysis.}
\label{ap:comp}
A primary concern in the Continual Test-Time Adaptation (CTTA) task is the requisite for computational efficiency, particularly in mobile computing applications such as autonomous driving vehicles, which are often constrained by limited computational resources. To address this issue, we present a detailed analysis of the computational costs of our proposed MoASE framework, including FLOPs, memory consumption, and model parameters with one Tesla A100, as shown in \cref{tab:comp}. It is observed that while MoASE incurs a slightly higher computational cost, the increase is modest with only 4.2\% in a number of parameters, and 4.0\% in FLOP—amounting to only a small percentage of the original model's resources. This is juxtaposed against a substantial accuracy improvement of 15.3\%, demonstrating the effectiveness of our approach in resource-constrained environments.

\section{More qualitative results.}
\label{ap:qu}

To further validate the effectiveness of our proposed method, we have expanded our evaluation by providing additional qualitative Class Activation Map (CAM) visualizations across nine distinct target domains. These visualizations are coupled with comparative analyses in the challenging Cityscapes-to-ACDC Cross-Domain Task-Aware Adaptation (CTTA) scenario. The segmentation outputs, meticulously derived through the CTTA process, are vividly illustrated at the top and bottom of \cref{fig:ap_vis}, respectively. These illustrations not only showcase the segmentation prowess of our method but also serve as a testament to the nuanced understanding of contextual discrepancies across different domains under th3cCTTA scenario.

The CAM visualizations, in particular, are pivotal in demonstrating the effectiveness of our MoASE method's activation decomposition strategy. By focusing on specific areas of interest within the images, these maps reveal how our method enhances domain-specific feature extraction, thereby facilitating more precise and contextually relevant insights into the underlying segmentation process. The top part of \cref{fig:ap_vis} displays these class attention maps, highlighting areas where our method has successfully concentrated its attention, correlating strongly with regions of significant semantic importance within each target domain.

At the bottom of \cref{fig:ap_vis}, our method's segmentation maps are displayed, offering a direct comparison with those generated by competing methodologies such as CoTTA, ViDA, and Continual-MAE across two additional target domains. Our method excels particularly in distinguishing subtle yet critical features such as the sidewalk from the road. This capability is emphasized by a white box in the visualizations, which draws attention to the precise delineation achieved by our segmentation approach. This distinction is crucial as it not only underscores our method's ability to produce more accurate segmentation outcomes but also highlights its robustness in effectively mitigating the impact of dynamic domain shifts.

Furthermore, in other categories beyond the highlighted examples, our method's segmentation outputs exhibit a high degree of resemblance to the ground truth. This close alignment leads to notable visual enhancements and provides a strong indication of the method's generalizability and effectiveness in handling diverse urban landscapes. The visual fidelity of these segmentation maps, coupled with the insightful CAM visualizations, collectively demonstrate the comprehensive capabilities of our proposed approach in adapting to and excelling within varied and dynamically changing environments.

\begin{figure*}[t]
% \vspace{-0.2cm}
\centering
\includegraphics[width=0.99\linewidth]{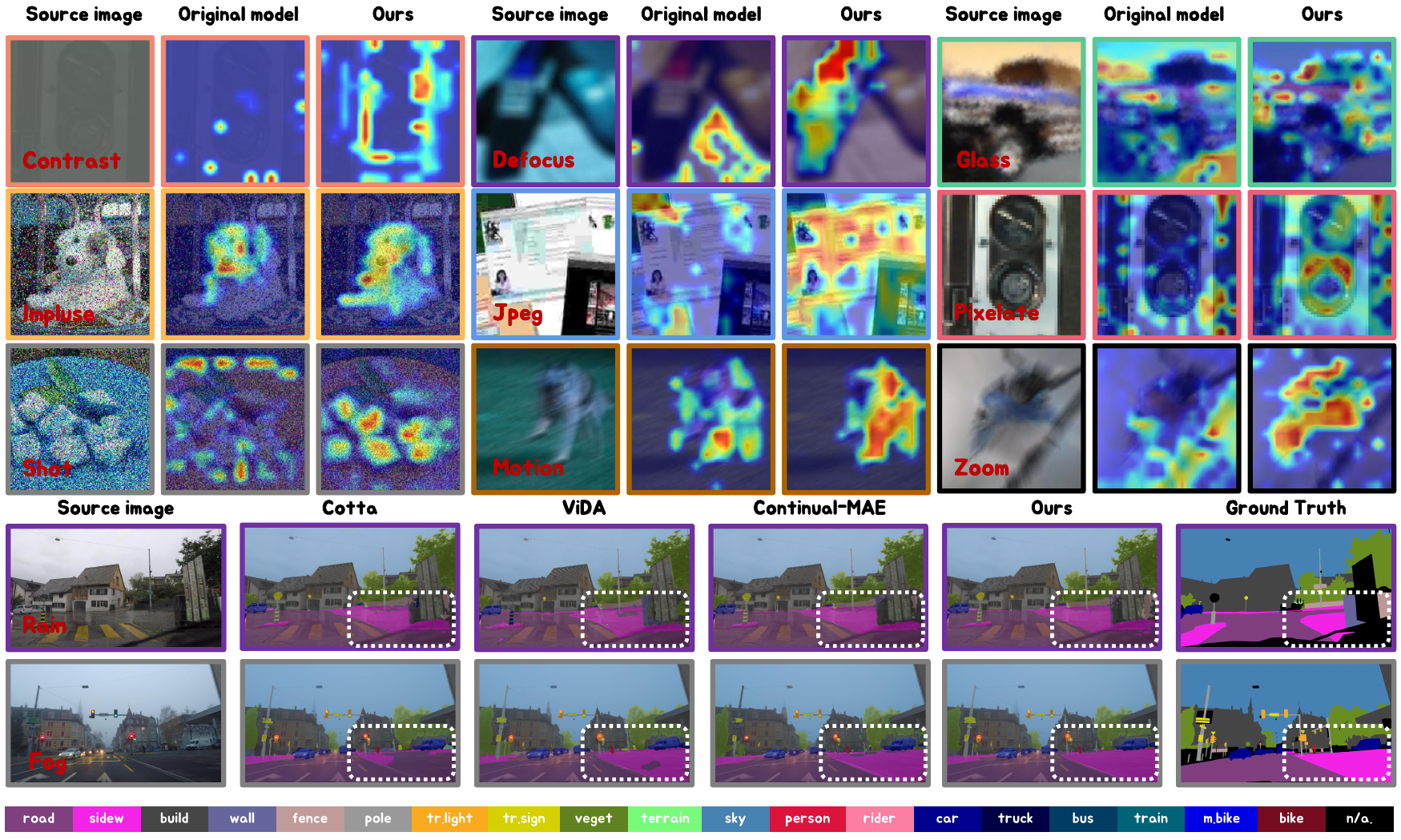}
% \vspace{-0.2cm}
\caption{The qualitative analysis of the CAM and the segmentation qualitative comparison of our method with previous SOTA methods on the ACDC dataset. }
% Our method could better segment different pixel-wise classes such as shown in the white box.}
\label{fig:ap_vis}
% \vspace{-0.3cm}
\end{figure*}

\end{document}